%% file: neurips_2026.tex
\documentclass{article}

\PassOptionsToPackage{numbers, compress}{natbib} 
\usepackage[preprint]{neurips_2026}
\usepackage{wrapfig}

\usepackage[utf8]{inputenc} 
\usepackage[T1]{fontenc}    
\usepackage{hyperref}       
\usepackage{url}            
\usepackage{booktabs}       
\usepackage{amsfonts}       
\usepackage{nicefrac}       
\usepackage{microtype}      
\usepackage{xcolor}         
\usepackage{algorithm}
\usepackage{algpseudocode}

\usepackage{bm}
\usepackage{CJKutf8}
\usepackage{amsmath}
\usepackage{amssymb}
\usepackage{mathtools}
\usepackage{amsthm}
\usepackage{adjustbox}
\usepackage{enumitem} 

\usepackage[table]{xcolor}

\usepackage{tabularx}
\usepackage{tcolorbox}
\tcbuselibrary{breakable,listingsutf8,skins}
\usepackage{listings}
\newtcblisting{promptbox}[1]{
  enhanced,
  breakable,
  colback=white,
  colframe=black,
  boxrule=0.8pt,
  title=#1,
  listing only,
  left=2mm,right=2mm,top=1mm,bottom=1mm,
  listing options={
    basicstyle=\ttfamily\small,
    columns=fullflexible,
    keepspaces=true,
    showstringspaces=false,
    breaklines=true,
    breakatwhitespace=false
  }
}

\usepackage{graphicx}

\usepackage{makecell}
\usepackage{subcaption}

\usepackage{caption}

\title{TrustRoboReward: Preference-Ordered Isotonic Score Editing for Multi-Paradigm Robot Reward Models}

\author{%
  \begin{minipage}{0.95\textwidth}
    \centering
    {\large\bfseries
      Yidong Wang\textsuperscript{1,*},
      Yan Zhan\textsuperscript{1,*},
      Ziteng Feng\textsuperscript{3},
      Zhenyu Cui\textsuperscript{4},
      Ziyi Zhou\textsuperscript{5},
      Renzhao Liang\textsuperscript{6},
      Jiaxuan Zhu\textsuperscript{4},
      Zilei Yang\textsuperscript{6},
      Yiran Zhao\textsuperscript{7},
      Zhongkuan Mao\textsuperscript{8},
      Bo Jia\textsuperscript{9},
      Hanchu Ni\textsuperscript{1},
      Chenggang Xie\textsuperscript{6},
      Biao Liu\textsuperscript{4},
      Yi Zhang\textsuperscript{2,\(\ddagger\)},
      Yong Dai\textsuperscript{2,\(\ddagger\)},
      Xiaozhu Ju\textsuperscript{2,\(\dagger\)},
      Wei Ye\textsuperscript{1,\(\dagger\)},
      Shikun Zhang\textsuperscript{1,\(\dagger\)},\\[0.45em]
    }
    {\normalsize\bfseries
      \textsuperscript{1}Peking University
      \qquad
      \textsuperscript{2}Beijing Innovation Center of Humanoid Robotics
      \qquad
      \textsuperscript{3}University of Science and Technology of China
      \qquad
      \textsuperscript{4}Southeast University
      \qquad
      \textsuperscript{5}Southern University of Science and Technology
      \qquad
      \textsuperscript{6}Beijing University of Aeronautics and Astronautics
      \qquad
      \textsuperscript{7}Beijing Language and Culture University
      \qquad
      \textsuperscript{8}Sichuan University
      \qquad
      \textsuperscript{9}Beijing University of Posts and Telecommunications}\\[0.25em]
    {\normalsize\normalfont
      \textsuperscript{*}Equal contribution
      \qquad
      \textsuperscript{\(\dagger\)}Corresponding authors
      \qquad
      \textsuperscript{\(\ddagger\)}Project leaders}
  \end{minipage}%
}

\begin{document}

\maketitle

\begin{abstract}
Reward models are a core bottleneck for reinforcement learning in embodied AI. Long-horizon robotic manipulation requires scalable vision-based feedback, which cannot be fully satisfied by manually crafted reward functions or task-specific human annotations. Existing open-source VLM-based reward judges like RoboReward only adopt a simple 1--5 trajectory progress scoring scheme, lacking pairwise preferences essential for modern RLHF, DPO and Bradley-Terry frameworks, while also failing to optimize agents' video scene understanding directly. Simply augmenting RoboReward with pairwise comparison and video-QA supervision causes inherent inconsistency between pairwise preferences and pointwise scores, introducing training noise and hurting downstream performance---an issue that aggregation methods such as TrustJudge cannot resolve. To address this, we propose TrustRoboReward\footnote{Our codebase and experimental results are available at: \url{https://github.com/TrustRoboreward/TrustRoboReward}.}, a multi-paradigm reward modeling framework equipped with the Preference-Ordered Isotonic Score Editing (POISE) module. We construct a unified four-paradigm dataset containing trajectory progress scoring (Score-A), video-QA answer quality scoring (Score-B), and their pairwise preference counterparts (Pair-A, Pair-B). We find pairwise labels align better with human judgment than pointwise scores, inspiring us to calibrate pointwise scores to avoid score-pair reversals against pairwise preferences. The POISE module rectifies pointwise scores and eliminates cross-paradigm reversal conflicts unresolved by TrustJudge. Theoretically, POISE provably reduces training-corpus score-pair reversal conflicts from 20.15\% to 0\%, whereas TrustJudge still retains 20.46\% reversal conflicts on the same training corpus. Evaluated on our multi-paradigm benchmark, the Qwen3-VL-4B model trained with POISE achieves an overall reward score of 77.96\%, nearly matching GPT-5-mini (78.09\%, gap 0.13\%) and outperforming the strongest RoboReward-4B baseline by 10.13\%. It also lifts test-time score-pair consistency to 71.90\%, exceeding both RoboReward-4B (57.26\%) and GPT-5-mini (68.09\%). Further integrating TrustJudge aggregation during inference boosts our model's overall score to 78.57\%, surpassing the proprietary GPT-5-mini teacher model.

\end{abstract}

\section{Introduction}
\label{sec:Introduction}
Despite advances in reinforcement learning (RL) for robot policy training~\citep{ma2024eureka,luo2024serl,liu2023visual,guo2025improving,kimrobot,ye2025reinforcement,ghasemipourself}, its large-scale deployment is still limited by the lack of scalable and generalizable reward models. Such models often rely on labor-intensive manual labeling or fragile handcrafted functions~\citep{ma2024eureka,christiano2017deep,ouyang2022instructgpt,lee2026roboreward}. Since robots operate in unstructured and dynamic environments, this problem is more challenging than in the mathematical/software domains that have automated validation mechanisms~\citep{shao2024deepseekmath,guo2025deepseek,guanrstar,pan2025training,jimenez2023swe,yang2024swe,wang2024openhands}. Vision-language models (VLMs)~\citep{liu2023visual,achiam2023gpt,chen2024internvl, bai2025qwen3} offer a solution, with the RoboReward~\citep{lee2026roboreward} framework serving as a typical example. This framework uses VLMs to assign task-progress rewards without per-task reward engineering, and through a consistent 1–5 scale discrete scoring, it trains Qwen~\citep{bai2025qwen3} VLM reward models that outperform larger closed-source VLMs.

However, RoboReward is limited along two key dimensions that we
address in this paper. \textbf{(1) Narrow single-paradigm supervision.}
RoboReward conducts training and evaluation via per-trajectory scoring on a 1-to-5 discrete rating scale. This single-formulation choice leaves two independent capability gaps in the reward signal. \textbf{Pointwise vs.\ pairwise.} RoboReward's outputs only take the form of pointwise scalar scores, while pairwise preference formats are widely adopted in modern reinforcement learning pipelines, including DPO~\citep{rafailov2023dpo}, RLHF~\citep{ouyang2022instructgpt} reward modeling, Bradley–Terry~\citep{bradley1952rank} and Elo ranking, as well as RLAIF-style~\citep{lee2023rlaif} preference optimization. Pairwise comparisons better conform to human annotation practices. In addition to public benchmarks like AlpacaEval~\citep{li2023alpacaeval} and Chatbot Arena~\citep{chiang2024chatbot}, our internal results presented in Figure~\ref{fig:intro-poise-calibration} also confirm the superior alignment of pairwise evaluation with human annotations. A pointwise-only reward model must manufacture pairwise preferences by
post-hoc score processing. However, as shown in Figure~\ref{fig:intro-poise-calibration}, real pairwise preferences conducted by humans/reward models often conflict with scalar score rankings. \textbf{Behavioral vs.\ grounded understanding.} Integer task-progress rates merely measure the robot’s execution performance, yet never directly verify whether the VLM comprehends the scene on which it reasons. Progress-only task completion scoring is inherently vulnerable to reward hacking. This metric only quantifies the robot’s final task execution outcome, with no mechanism to validate whether the underlying vision-language model has true grounded understanding of the target scene. This critical blind spot enables models to achieve high scores via spurious correlations instead of semantically correct task reasoning, as shown in Appendix \ref{fig:case-study-frames}, a hidden failure mode completely undetectable in outcome-only evaluation. \textbf{(2) Uncontrolled Cross-Paradigm Training Noise from Supervision Inconsistency.}
A straightforward fix for the first limitation is to add pairwise comparison and grounded-QA supervision to existing completion-focused score supervision for joint training. Pointwise scoring labels and pairwise ranking labels, even when deployed exclusively to evaluate task completion, exhibit the fundamental inconsistency revealed in TrustJudge\citep{wang2026trustjudge}, and as a result, regularly produce conflicting judgments on the same sample: a sample earning a high completion score in pointwise assessment may be marked as the losing candidate in direct pairwise comparison. Critically, TrustJudge mitigates cross-paradigm inconsistency via probabilistic aggregation to recalibrate raw scoring outputs. This approach can only alleviate inherent cross-paradigm inconsistency rather than eradicating it entirely. If left unaddressed in the training corpus, these conflicting signals generate opposing training signals during training, thus preventing the reward model from converging to a stable, self-consistent judgment function.

\begin{figure}[H]
    \centering
    \includegraphics[width=\linewidth]{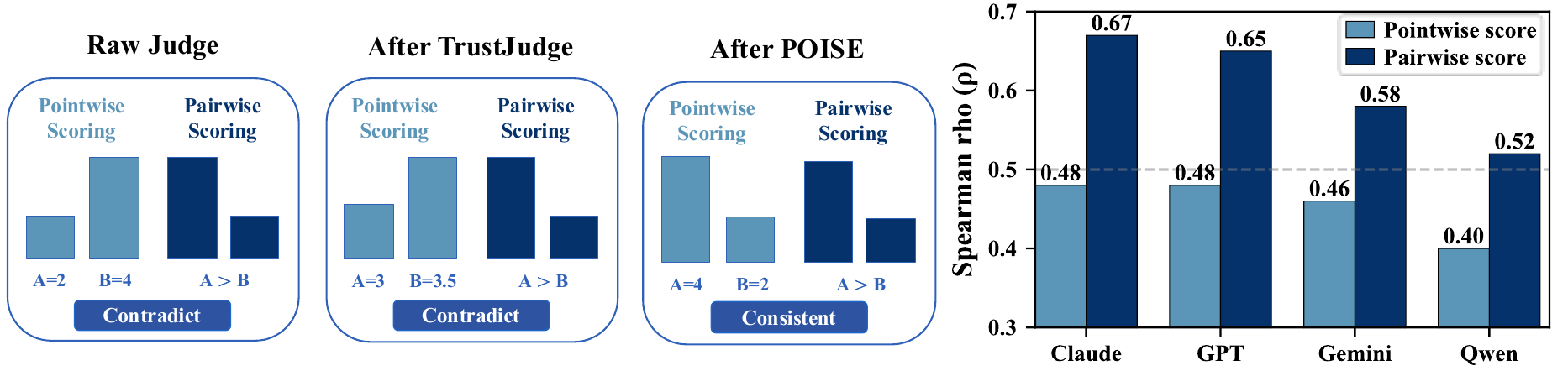}
    \caption{Calibration of score labels using POISE with pair preference signals corrects inconsistencies. Spearman's $\rho > 0.5$ indicates strong positive correlation between LLM and human expert rankings. Pairwise scores align better with human preferences than pointwise scores. Gemini, GPT, Claude, and Qwen denote Gemini 3.1 Pro, GPT-5-mini, Claude Opus 4.6, and Qwen3-VL-32B respectively.}
    \label{fig:intro-poise-calibration}
\end{figure}

These limitations are closely coupled and cannot be treated separately. Deriving pairwise preferences from score-only reward outputs naturally gives rise to cross-paradigm contradictions, and introducing pairwise supervision further brings cross-paradigm noise that must be regulated to enable robust multi-paradigm reward modeling. To address both limitations, we propose \textbf{TrustRoboReward}, a multi-paradigm reward modeling framework consisting of two complementary components. \textbf{Multi-paradigm supervision} (§\ref{sec:method-mpd}) queries a teacher to construct four complementary supervision paradigms---trajectory progress scoring (Score-A), video-QA answer quality scoring (Score-B), and their pairwise counterparts (Pair-A and Pair-B)---on a shared set of robot video samples, simultaneously closing the pointwise-versus-pairwise gap and the behavioral-versus-grounded gap. \textbf{POISE} (Preference-Ordered Isotonic Score Editing; §\ref{sec:method-poise}) then eliminates the cross-paradigm training noise that multi-paradigm supervision inevitably introduces: it treats pairwise labels as a target partial order and projects raw pointwise scores onto the monotone cone of that order via the Pool-Adjacent-Violators Algorithm, provably reducing training-corpus score-pair reversal conflicts to exactly $0\%$, which is a training-time guarantee that TrustJudge's probabilistic aggregation cannot provide.

We evaluate TrustRoboReward on a comprehensive multi-paradigm benchmark spanning all four supervision paradigms. With a Qwen3-VL-4B backbone, our model achieves an overall reward score of $77.96\%$, nearly matching the proprietary GPT-5-mini\citep{singh2025openai} teacher ($78.09\%$, a gap of only $0.13\%$) while outperforming the strongest open-source RoboReward-4B baseline by $10.13\%$. On test-time score-pair consistency, our model reaches $71.90\%$, exceeding both RoboReward-4B ($57.26\%$) and GPT-5-mini ($68.09\%$). Combining our model with TrustJudge at inference time further lifts the overall score to $78.57\%$, surpassing the proprietary teacher. We additionally validate TrustRoboReward as a downstream reward signal for embodied policy optimization on PAIBench-G\citep{zhou2025pai}, where it consistently outperforms RoboReward in guiding robot manipulation.

In conclusion, we contribute: \textbf{(i)} a four-paradigm supervision corpus (Score-A/B, Pair-A/B) that simultaneously closes the pointwise-versus-pairwise and behavioral-versus-grounded capability gaps; \textbf{(ii)} POISE, a linear-time isotonic denoising algorithm that provably reduces training-corpus score-pair reversal conflicts to exactly $0\%$ and monotonically improves score quality (Theorem~\ref{thm:poise-quality}); \textbf{(iii)} an open-weight Qwen3-VL-4B reward model that nearly matches GPT-5-mini (gap $0.13\%$) and surpasses it with TrustJudge aggregation.


\section{Related Work}
\label{sec:related_work}

\textbf{VLMs as reward judges.}
VLMs such as LLaVA~\citep{liu2023visual, lillava} and
Qwen3-VL~\citep{bai2025qwen3} have enabled the VLM-as-a-Judge
paradigm~\citep{zheng2023judging,lee2024prometheus,chen2024mllm,xiong2025llava}, in which a strong
generalist model scores or ranks candidate outputs in lieu of human
annotators. A growing body of work casts such judges as
\emph{generative} reward models that emit chain-of-thought
rationales together with scalar
verdicts~\citep{mahan2024genrm,cao2024compassjudger}, and as
zero-shot visual reward models for downstream
RL~\citep{rocamonde2024vlmrm, wang2024rl, baumli2023vision} or video
generation~\citep{xu2026visionreward, liuimproving}.
Dedicated reward-model benchmarks such as
RewardBench~\citep{lambert2024rewardbench, yasunaga2025multimodal,li2025vl} systematically evaluate
such judges and reveal persistent reliability gaps.
In robotics specifically, RoboReward~\citep{lee2026roboreward}
distills a proprietary teacher into Qwen-based judges emitting a
$1$-to-$5$ trajectory-progress score. We follow this distillation
recipe but extend supervision beyond single-paradigm pointwise
scoring to pairwise preferences and grounded video-QA, and
explicitly resolve the cross-paradigm inconsistency that naive
multi-paradigm distillation introduces.

\textbf{Preference learning and score--comparison consistency.}
Modern alignment pipelines consume \emph{pairwise} preferences:
RLHF~\citep{ouyang2022instructgpt} and DPO~\citep{rafailov2023dpo, meng2024simpo}
train reward models under Bradley--Terry~\citep{bradley1952rank}
assumptions, and pairwise judgments align better with humans than
Likert scoring~\citep{zheng2023judging,chen2024mllm, liualigning}---a trend
Figure~\ref{fig:intro-poise-calibration} also confirms for embodied
video. The recent shift toward AI-generated preference labels
(RLAIF)~\citep{lee2023rlaif, bai2022constitutional} and self-improving
judges~\citep{mahan2024genrm, wu2025meta} further amplifies the value of
\emph{consistent} preference signals at scale. When the same judge
produces both pointwise and pairwise labels, however, the two
channels frequently disagree:
TrustJudge~\citep{wang2026trustjudge} quantifies this
score--comparison inconsistency and smooths it at \emph{inference}
via probabilistic aggregation. We address the same diagnosis from
the opposite end, treating pairwise labels as a partial-order
constraint on the \emph{training} corpus and projecting pointwise
scores onto its monotone cone via the Pool-Adjacent-Violators
algorithm for isotonic
regression~\citep{ayer1955empirical}. This training-time
correction is provably consistent ($\eta_{\mathrm{cross}}{=}0$)
and complements TrustJudge at inference (§\ref{tab:ablation:tj}).

\section{Method}
\label{sec:method}

\newcommand{\ncross}{\eta_{\mathrm{cross}}}


We address the two limitations raised in
Section~\ref{sec:Introduction} with two complementary components. To close the narrow single-paradigm gap (limitation~1),
\emph{multi-paradigm supervision} (§\ref{sec:method-mpd})
queries a teacher to construct supervision under four paradigms—two pointwise paradigms, namely trajectory progress scoring (Score-A) and video-QA answer quality scoring (Score-B), and two pairwise paradigms, namely their pairwise preference forms for dual-video preference (Pair-A) and dual-answer preference (Pair-B)—on a shared set of RoboReward video samples. To eliminate the cross-paradigm training
noise this supervision inevitably introduces (limitation~2),
POISE (\emph{Preference-Ordered Isotonic Score Editing};
§\ref{sec:method-poise}) reads the pairwise labels as a target partial
order, projects the raw pointwise scores onto the monotone cone
of that order by isotonic correction with the Pool-Adjacent-%
Violators Algorithm (PAVA)~\citep{ayer1955empirical}, and feeds the resulting self-%
consistent corpus into multi-paradigm training.

\subsection{Multi-paradigm Supervision}
\label{sec:method-mpd}

To overcome RoboReward's narrow single-paradigm supervision, we
distill four complementary reward paradigms from the teacher
query onto a shared item pool, together spanning both capability
axes identified in Section~\ref{sec:Introduction}. Score-A asks the
teacher to rate, on a $1$-to-$5$ scale, a single trajectory's
task-progress under a given prompt, matching RoboReward's
original scoring protocol. Score-B rates video-QA candidate answers on a consistent scale, prompting VLMs to recognize objects, actions and underlying states rather than relying on surface-level cues. Pair-A compares two trajectories on the same task and
emits a label in $\{\mathrm{win},\mathrm{lose},\mathrm{tie}\}$,
providing the pairwise format that DPO and RLHF reward modeling
consume natively. Pair-B compares two candidate answers for the video-QA instance with a shared label vocabulary.

More crucially, the four paradigms share the same set of
RoboReward video samples, so their labels can be compared
directly on the same samples. We organize samples into groups
by task context: each group consists of candidate trajectories
sharing a common task instruction (Score-A / Pair-A) or
candidate answers sharing a common video (Score-B / Pair-B).
All supervision is defined within a group and never 
across groups, so every pairwise sample aligns with the score
labels of its two endpoints. And these two endpoints have shared-grouping property that
makes §\ref{sec:method-poise}'s cross-paradigm consistency
constraint well-defined.

All four paradigms use a uniform Point-then-Explain output
format (score or comparison token first, then rationale) and are
consumed in a unified two-stage training pipeline (SFT\citep{weifinetuned} then RL)
under matched training configurations, so label-variant
comparisons (raw vs corrected) isolate the label effect rather
than differences in training recipe. The full prompt specification is provided in Appendix~\ref{app:prompts}, including the shared system prompt in Appendix~\ref{app:prompt_system}, the Score-A template in Appendix~\ref{app:prompt_scorea}, the Score-B template in Appendix~\ref{app:prompt_scoreb}, the Pair-A template in Appendix~\ref{app:prompt_paira}, and the Pair-B template in Appendix~\ref{app:prompt_pairb}.

Multi-paradigm supervision therefore resolves limitation~1 of
Section~\ref{sec:Introduction}: supervision under four
complementary paradigms on the same underlying video samples
simultaneously closes the pointwise-versus-pairwise and
behavioral-versus-grounded capability gaps that RoboReward lefts open.

\subsection{POISE: Preference-Ordered Isotonic Score Editing}
\label{sec:method-poise}

The pairwise annotations and pointwise scores are often inconsistent on the same items. To address this mismatch, we propose POISE (Preference-Ordered Isotonic Score Editing) to recalibrate pointwise scores toward weak-order consistency, i.e., no pointwise score ranking contradicts the pairwise preference direction. We formalize the cross-paradigm noise, show why no label-local denoiser can reduce it, derive the partial-order isotonic correction, and give pseudocode with a worked example.

\textbf{Cross-paradigm noise rate.} We define the cross-paradigm noise rate $\ncross$ on a single group as the fraction
of jointly-labeled pairs whose score-implied ordering disagrees
with the pairwise label. Fix a group with $m$ items, score labels
$s_i \in \{1,\dots,5\}$, and pairwise labels $v_{ij} \in
\{\mathrm{win},\mathrm{lose},\mathrm{tie}\}$ (item $i$ wins,
loses, or ties with $j$); the same problem applies once per axis
(Score-A$+$Pair-A and Score-B$+$Pair-B). Formally, $\ncross \;=\;
\Pr_{(i,j)}\!\bigl[(v_{ij} {=} \mathrm{win} \wedge s_i {<} s_j)\,\vee\,
(v_{ij} {=} \mathrm{lose} \wedge s_i {>} s_j)\bigr]$.

Because $\ncross$ couples $\{s_i\}$ and $\{v_{ij}\}$, label-local strategies fundamentally fail to alleviate this inconsistency. This category includes TrustJudge\citep{wang2026trustjudge}, which computes expected scores from probabilistic teacher outputs, and majority voting, which identifies the mode among predictions. Neither approach can align the derived scores with pairwise decisions on the same items.

\begin{algorithm}[h]
\caption[POISE (Preference-Ordered Isotonic Score Editing).]{POISE (Preference-Ordered Isotonic Score Editing).\protect\\\textit{Example ($m{=}4$): raw scores $(s_1, s_2, s_3, s_4) = (2, 1, 5, 3)$; pairwise labels $v_{24} = v_{41} = v_{13} = \mathrm{lose}$ (the remaining labels follow by transitivity), yielding the score chain {\normalfont\large\textcircled{\small 2}} $\prec$ {\normalfont\large\textcircled{\small 4}} $\prec$ {\normalfont\large\textcircled{\small 1}} $\prec$ {\normalfont\large\textcircled{\small 3}}(worst to best).}}
\label{alg:poise}
\begin{algorithmic}[1]
\Require per-group raw score label
$s_1,\dots,s_m \in \{1,\dots,5\}$;
per-group pairwise label
$\{v_{ij}\}_{i,j=1}^{m} \in
\{\mathrm{win},\mathrm{lose},\mathrm{tie}\}$ inducing a valid partial order
\Ensure corrected score label
$\hat{s}_1,\dots,\hat{s}_m \in \{1,\dots,5\}$
($\ncross = 0$ reversal-conflict on the training corpus)
\State $(r_1,\dots,r_m) \gets
\textsc{ReindexByPairwiseOrder}(\{v_{ij}\},\, s_1,\dots,s_m)$
\Comment{$(r_1,\dots,r_m)$: chain order along $v$ from worst to best; $r_i$ is the original score index at chain position $i$}
\Statex \hspace{1.6em}{\footnotesize\color{gray}\textit{// Example: $v$ induces chain order $(r_1,r_2,r_3,r_4) = (2,4,1,3)$, so $(s_{r_1},s_{r_2},s_{r_3},s_{r_4}) = (s_2,s_4,s_1,s_3) = (1,3,2,5)$.}}
\State $(\hat{s}'_1,\dots,\hat{s}'_m) \;\gets\;
\displaystyle\arg\!\min_{s' \in \mathbb{R}^m}\,
\sum_{i=1}^{m}\bigl(s'_i - s_{r_i}\bigr)^{2}
\ \text{s.t.}\ s'_1 \!\le\! s'_2 \!\le\! \dots \!\le\! s'_m$
\Comment{minimize score change subject to chain-monotonicity; $\hat{s}'_i$ may be non-integer. Detailed implementation is given by Eq.~\eqref{eq:pav-block}.}
\Statex \hspace{1.6em}{\footnotesize\color{gray}\textit{// Example: PAV pools chain positions $2,3$ to mean $2.5$; $(\hat{s}'_1,\hat{s}'_2,\hat{s}'_3,\hat{s}'_4) = (1,2.5,2.5,5)$.}}
\Statex \hspace{1.6em}{\footnotesize\color{gray}\textit{// each $\hat{s}'_i$ is the constrained optimum corresponding to the score at chain position $i$, i.e.,
$(\hat{s}'_1,\hat{s}'_2,\hat{s}'_3,\hat{s}'_4)$ corresponds to $(s_2,s_4,s_1,s_3)$}}
\For{$i = 1, \dots, m$}
  \State $\hat{s}_{r_i} \gets \max\!\bigl(1,\,\min(5,\,
  \lfloor \hat{s}'_i + \tfrac{1}{2} \rfloor)\bigr)$
  \Comment{round-half-up, then clip back to $\{1,\dots,5\}$}
\EndFor
\Statex \hspace{1.6em}{\footnotesize\color{gray}\textit{// Example: round-half-up writes $(\hat{s}_2,\hat{s}_4,\hat{s}_1,\hat{s}_3) = (1,3,3,5)$ in chain order from worst to best.}}
\State \Return $(\hat{s}_1,\dots,\hat{s}_m)$
\Statex \hspace{1.6em}{\footnotesize\color{gray}\textit{// Example: original scores are $(s_1,s_2,s_3,s_4) = (2,1,5,3)$; corrected scores become $(\hat{s}_1,\hat{s}_2,\hat{s}_3,\hat{s}_4) = (3,1,5,3)$ in the original score order; the corrected scores satisfy $\hat{s}_2 \le \hat{s}_4 \le \hat{s}_1 \le \hat{s}_3$.}}
\end{algorithmic}
\end{algorithm}

\textbf{Partial-order isotonic correction.} POISE fixes the pairwise labels and rewrites the scores under a partial-order constraint. Declaring $i \prec j$ when $v_{ij}=\mathrm{lose}$ (and $\mathrm{tie}$ as equivalence) turns the pairwise labels into a partial order, and we seek scores close to $\{s_i\}$ that are monotone along this order. The only precondition is that $\{v_{ij}\}$ induces a valid partial order; since RoboReward groups contain at most three items, cycles occur in $<1\%$ of groups and are fixed by expert re-annotation.

With a valid $\{v_{ij}\}$ in hand, we construct the chain order
$r=(r_1,\dots,r_m)$, i.e., a linear extension of the induced partial order,
where $r_i$ is the original score index at chain position $i$ (worst to best). We then solve chain isotonic regression.
Any assignment monotone along this chain is monotone under the
original partial order, so this is a conservative sufficient
constraint for eliminating reversal conflicts:
\begin{equation}
\label{eq:poise-qp}
(\hat{s}'_1, \dots, \hat{s}'_m) \;=\;
\arg\min_{s' \in \mathbb{R}^m}
\sum_{i=1}^{m} \bigl(s'_i - s_{r_i}\bigr)^2
\quad \text{s.t.} \quad
s'_1 \le s'_2 \le \dots \le s'_m,
\end{equation}
which the Pool-Adjacent-Violators iteration \citep{ayer1955empirical} solves exactly in time
complexity of $O(m)$. The solution partitions $\{1,\dots,m\}$ into $K$ consecutive
blocks $B_1,\dots,B_K$ along the chain. Each block pools a run
of chain positions whose raw scores would otherwise violate
the non-decreasing constraint of \eqref{eq:poise-qp}, and all
positions in the block share a single corrected score equal to
the mean of the pooled raw scores. Corrected scores rise between adjacent blocks, so each block is a maximal contiguous segment sharing one corrected value:
\begin{equation}
\label{eq:pav-block}
\hat{s}'_i \;=\; \mu_k \;\equiv\; \frac{1}{|B_k|}\sum_{j \in B_k} s_{r_j},
\qquad i \in B_k.
\end{equation}
With strictly increasing block means $\mu_1 < \mu_2 < \dots < \mu_K$, we apply a single left-to-right pass. Each chain position is initialized as a single block. Whenever the newly added block has a smaller mean than its predecessor, we merge the two blocks and replace them with their weighted-average block mean until adjacent block means are nondecreasing. The real-valued outputs are then clipped and mapped back to the integer score range $\{1,\dots,5\}$, and written to item $r_i$ by $\hat{s}_{r_i}=\max\!\bigl(1,\min(5,\lfloor \hat{s}'_i + \tfrac{1}{2} \rfloor)\bigr)$, preserving the induced chain order and hence the partial order.

Algorithm~\ref{alg:poise} instantiates POISE as a per-group
linear-time step. The
procedure runs independently across groups and twice per
dataset(Score-A$+$Pair-A and
Score-B$+$Pair-B).

\textbf{Data splits.} POISE induces six splits for downstream training. $\mathcal{P}_1$ and $\mathcal{S}_1$ are the raw pairwise and score splits from §\ref{sec:method-mpd}. We then partition them as $\mathcal{P}_1=\mathcal{P}_2\cup\mathcal{P}_3$ and $\mathcal{S}_1=\mathcal{S}_2\cup\mathcal{S}_3$,
where the unions are disjoint. The two sides are produced by different mechanisms:
\textbf{(score side)} $\mathcal{S}_2,\mathcal{S}_3$ come from POISE score editing (pointwise labels unchanged in $\mathcal{S}_2$ and rewritten in $\mathcal{S}_3$);
\textbf{(pairwise side)} $\mathcal{P}_2,\mathcal{P}_3$ come from pairwise-order repair (cycle/transitivity conflicts corrected by expert re-annotation, with unchanged-vs-rewritten labels split into $\mathcal{P}_2$ vs.\ $\mathcal{P}_3$).
Accordingly, $\mathcal{S}_2,\mathcal{P}_2$ are SFT-admissible (teacher rationale still valid), while $\mathcal{S}_3,\mathcal{P}_3$ are RL-only residuals (labels/scores rewritten, so the original rationale may be invalidated).

Our default deployed recipe (\textbf{Ours}) cleans only the score side ($\mathcal{S}_1\!\to\!\mathcal{S}_2,\mathcal{S}_3$) and keeps $\mathcal{P}_1$ intact, because pairwise cycle/transitivity conflicts affect fewer than $1\%$ of groups and are already repaired in pairwise-order preprocessing; detailed routing definitions are listed in Appendix~\ref{app:stats}.


\newtheorem{theorem}{Theorem}
\newtheorem{lemma}{Lemma}
\newtheorem{remark}{Remark}

\textbf{Theoretical guarantee.} We now provide a formal justification for the effectiveness of POISE. Under pairwise-order consistency, POISE's isotonic projection monotonically improves score quality.The detailed proof is provided in Appendix~\ref{app:proof}. The main conclusion is stated below.

\begin{theorem}[Isotonic Alignment Improves Score Quality]
\label{thm:poise-quality}
Let $g\in\{1,\dots,5\}^m$ be the ground-truth golden scores and $s\in\{1,\dots,5\}^m$ the raw noisy scores for a group of $m$ items. Let $\preceq$ be the partial order induced by pairwise labels that is consistent with $g$ (i.e., $i\preceq j \iff g_i\le g_j$). Let $\hat{s}'\in\mathbb{R}^m$ denote the real-valued chain-isotonic regression of $s$ on a linear extension $r$ of $\preceq$, solved by PAVA \emph{before} quantization. Then
\begin{equation}
\|\hat{s}'-g\|_2 \;\le\; \|s-g\|_2.
\label{eq:quality-l2}
\end{equation}
Furthermore, if $s$ violates the order $\preceq$ on at least one comparable pair, the improvement is strict in the squared sense:
\begin{equation}
\|\hat{s}'-g\|_2^2 \;\le\; \|s-g\|_2^2 - \|\hat{s}'-s\|_2^2 \;<\; \|s-g\|_2^2.
\label{eq:quality-strict}
\end{equation}
\end{theorem}

POISE therefore resolves limitation~2 of Section~\ref{sec:Introduction}: by construction, the training-corpus reversal-conflict rate $\ncross = 0$ after one $O(m)$ sweep per group with no additional teacher queries; together with the quality guarantee in Theorem~\ref{thm:poise-quality}, this shows that POISE removes cross-paradigm reversal inconsistencies while preserving (indeed improving) score quality. Downstream SFT/RL training remains unchanged from the Vanilla-SFT baseline, so any improvement over Vanilla-SFT isolates the effect of our cross-paradigm correction.

\input{exp}

\section{Conclusion}

We introduced TrustRoboReward, a multi-paradigm reward modeling framework for embodied AI that extends RoboReward-style supervision from pointwise trajectory-progress scoring to four complementary paradigms: trajectory scoring, video-QA scoring, and their pairwise counterparts. The central challenge is that multi-paradigm supervision exposes score--preference conflicts, which create noisy and contradictory training signals. To address this, POISE treats pairwise labels as ordering constraints and edits pointwise scores through isotonic projection, producing a self-consistent corpus with provably zero training-time score--pair reversal conflicts under valid pairwise orders. Experiments show that a Qwen3-VL-4B model trained with TrustRoboReward nearly matches GPT-5-mini, substantially improves over open-weight RoboReward baselines, and achieves stronger score--pair consistency. Inference-time TrustJudge further complements POISE and surpasses the teacher overall. Finally, PAIBench-G optimization shows that our reward model provides a stronger embodied task signal while preserving visual quality, suggesting that consistent multi-paradigm supervision is a practical path toward scalable robot reward modeling. We discuss the limitations of the current framework in Appendix~\ref{app:limitations}.

\begin{ack}
Use unnumbered first level headings for the acknowledgments. All acknowledgments
go at the end of the paper before the list of references. Moreover, you are required to declare
funding (financial activities supporting the submitted work) and competing interests (related financial activities outside the submitted work).
More information about this disclosure can be found at: \url{https://neurips.cc/Conferences/2026/PaperInformation/FundingDisclosure}.

Do {\bf not} include this section in the anonymized submission, only in the final paper. You can use the \texttt{ack} environment provided in the style file to automatically hide this section in the anonymized submission.
\end{ack}

\bibliography{main}
\bibliographystyle{unsrt}


\appendix
\section{Limitations}
\label{app:limitations}

TrustRoboReward relies on the assumption that pairwise preferences provide a more reliable ordering signal than pointwise scores. This assumption is supported by our human-correlation study and by prior observations that pairwise judgments often align better with human preferences, but it may not hold uniformly across all robot tasks or video-QA settings. When pairwise labels are systematically biased or incorrect, POISE will preserve the biased order and may propagate the error into the edited pointwise scores. Thus, the quality of pairwise supervision remains a central factor in the reliability of the final reward model.

POISE also assumes that the pairwise labels within each group induce a valid partial order. In our dataset, cycles are rare because each RoboReward group contains at most three items, and the remaining cycle or transitivity conflicts are repaired by expert re-annotation. However, larger candidate groups may produce more complex and frequent ordering conflicts. Extending POISE to large-scale noisy preference graphs without manual cycle repair is an important direction for future work.

Our theoretical guarantee applies to the real-valued isotonic projection before integer quantization and under the assumption that the induced pairwise order is consistent with the golden scores. The final training labels are discrete scores in $\{1,\dots,5\}$, so coordinate-wise rounding can introduce a small additional error. Although this effect is bounded and empirically small in our setting, a fully discrete isotonic projection may be preferable when exact discrete-label guarantees are required.

Our evaluation focuses on RoboReward-derived robot video samples, video-QA supervision, and a sampled subset of PAIBench-G reward-guided optimization tasks. These experiments cover diverse embodied settings, but they do not exhaust all manipulation regimes, long-horizon tasks, real-robot deployment conditions, or safety-critical failure modes. In particular, our downstream validation is based on video generation and human preference comparison rather than closed-loop policy execution on physical robots. Future work should evaluate whether the improved reward signal transfers to real robotic control, interactive policy learning, and out-of-distribution environments.

Finally, TrustRoboReward introduces additional data-construction and training cost compared with single-paradigm distillation. Building the four-paradigm corpus requires teacher queries, consistency analysis, and limited expert re-annotation for invalid pairwise orders, and training the final models requires SFT and GRPO\citep{shao2024deepseekmath} stages. While these costs are modest relative to collecting large-scale human reward annotations, they may still limit rapid iteration or deployment in resource-constrained settings.

\section{Proof of Theorem~\ref{thm:poise-quality}}
\label{app:proof}

For a group of $m$ items, let $r=(r_1,\dots,r_m)$ be the chain order (linear extension of $\preceq$) constructed by POISE, indexing items from worst to best. Define the \emph{chain monotone cone}
\begin{equation}
\mathcal{C}_r \;=\; \bigl\{x\in\mathbb{R}^m : x_{r_1}\le x_{r_2}\le\cdots\le x_{r_m}\bigr\}.
\label{eq:chain-cone}
\end{equation}
Because $r$ respects $\preceq$ (if $i\preceq j$ then $i$ appears no later than $j$ in $r$), any vector monotone along $r$ is automatically monotone under $\preceq$. By the order-consistency assumption (A2), the golden scores satisfy $g_{r_1}\le\cdots\le g_{r_m}$, hence
\begin{equation}
g\in\mathcal{C}_r.
\label{eq:g-in-cone}
\end{equation}

POISE computes the real-valued isotonic regression
\begin{equation}
\hat{s}' \;=\; \arg\min_{x\in\mathcal{C}_r}\|x-s\|_2^2,
\label{eq:pava-projection}
\end{equation}
which is exactly the metric projection of $s$ onto the closed convex cone $\mathcal{C}_r$, denoted $\hat{s}'=P_{\mathcal{C}_r}(s)$.

\begin{lemma}[Projection onto Closed Convex Set]
\label{lem:projection}
Let $\mathcal{C}\subseteq\mathbb{R}^m$ be a non-empty closed convex set. For any $y\in\mathbb{R}^m$ and any $z\in\mathcal{C}$,
\begin{equation}
\|P_{\mathcal{C}}(y)-z\|_2^2 \;\le\; \|y-z\|_2^2 - \|y-P_{\mathcal{C}}(y)\|_2^2.
\end{equation}
\end{lemma}

\begin{proof}
By the variational characterization of metric projections onto convex sets,
\begin{equation}
\langle y-P_{\mathcal{C}}(y),\, z-P_{\mathcal{C}}(y)\rangle \;\le\; 0 \qquad\forall z\in\mathcal{C}.
\label{eq:variational}
\end{equation}
Decomposing $y-z = \bigl(y-P_{\mathcal{C}}(y)\bigr) + \bigl(P_{\mathcal{C}}(y)-z\bigr)$ yields
\begin{align}
\|y-z\|_2^2
&= \|y-P_{\mathcal{C}}(y)\|_2^2 + \|P_{\mathcal{C}}(y)-z\|_2^2 + 2\langle y-P_{\mathcal{C}}(y),\, P_{\mathcal{C}}(y)-z\rangle \nonumber\\
&= \|y-P_{\mathcal{C}}(y)\|_2^2 + \|P_{\mathcal{C}}(y)-z\|_2^2 - 2\langle y-P_{\mathcal{C}}(y),\, z-P_{\mathcal{C}}(y)\rangle \nonumber\\
&\ge \|y-P_{\mathcal{C}}(y)\|_2^2 + \|P_{\mathcal{C}}(y)-z\|_2^2,
\label{eq:expand}
\end{align}
where the inequality uses \eqref{eq:variational}. Rearranging \eqref{eq:expand} gives the claim.
\end{proof}

\begin{proof}[Proof of Theorem~\ref{thm:poise-quality}]
Apply Lemma~\ref{lem:projection} with $y=s$, $z=g$, and $\mathcal{C}=\mathcal{C}_r$. Using \eqref{eq:g-in-cone} and \eqref{eq:pava-projection} we obtain
\begin{equation}
\|\hat{s}'-g\|_2^2 \;\le\; \|s-g\|_2^2 - \|\hat{s}'-s\|_2^2 \;\le\; \|s-g\|_2^2,
\end{equation}
which proves \eqref{eq:quality-strict}. Taking square roots gives \eqref{eq:quality-l2}. 

If $s$ violates $\preceq$ on some comparable pair, then $s$ is not monotone along the chain $r$, i.e., $s\notin\mathcal{C}_r$. Because the projection is unique and $\hat{s}'\in\mathcal{C}_r$, we have $\hat{s}'\neq s$ and thus $\|\hat{s}'-s\|_2^2>0$, yielding the strict inequality.
\end{proof}

\begin{remark}[Quantization to Integer Labels]
\label{rem:quantization}
POISE quantizes $\hat{s}'$ back to $\{1,\dots,5\}^m$ by coordinate-wise round-half-up, yielding the final discrete scores $\hat{s}$. While rounding is not a metric projection onto a convex set, the per-coordinate error is bounded by $0.5$. Consequently,
\begin{equation}
\|\hat{s}-g\|_2 \;\le\; \|\hat{s}'-g\|_2 + \frac{\sqrt{m}}{2}.
\end{equation}
In our setting $m\le 3$, so the additional error is at most $0.87$. In practice, the reduction $\|s-\hat{s}'\|_2^2$ from isotonic regression (which is positive exactly when cross-paradigm noise exists) overwhelmingly dominates this small quantization overhead. Hence $\hat{s}$ remains substantially closer to $g$ than $s$ in all empirical regimes. 

If a fully discrete theoretical guarantee is desired, one may replace round-half-up by an exact projection onto the discrete monotone set $\mathcal{C}_r\cap\{1,\dots,5\}^m$; this is computationally trivial for $m\le 3$ and preserves the validity of Theorem~\ref{thm:poise-quality} with $\hat{s}$ in place of $\hat{s}'$.
\end{remark}

\section{Detailed statistical distributions of training and evaluation splits}
\label{app:stats}
We consider four recipes that exhaust the natural choices of which side(s) of the corpus to clean, each specified by its (SFT data, RL data) tuple over $\mathcal{S}_{1,2,3}$ and $\mathcal{P}_{1,2,3}$:
\begin{itemize}
\item $\bigl(\mathcal{S}_1\cup\mathcal{P}_1,\;\text{--}\bigr)$
      --- the raw single-decode distillation corpus is used directly
      for SFT, with no DAG/PAVA cleaning and no RL stage. SFT
      $181{,}692$, RL $0$.
\item $\bigl(\mathcal{S}_2\cup\mathcal{P}_2,\;\mathcal{S}_3\cup\mathcal{P}_3\bigr)$
      --- both pointwise and pairwise sides are cleaned; the residual
      conflict samples on both sides are routed into RL. SFT
      $173{,}335$ ($\mathcal{S}_2 / \mathcal{P}_2$ counts
      $41{,}682 + 42{,}517$ pointwise plus $43{,}816 + 45{,}320$
      pairwise), RL $8{,}357$ ($\mathcal{S}_3 / \mathcal{P}_3$ counts
      $3{,}390 + 2{,}555$ pointwise plus $712 + 1{,}700$ pairwise).
\item $\bigl(\mathcal{S}_2\cup\mathcal{P}_1,\;\mathcal{S}_3\bigr)$
      --- our deployed recipe \textbf{Ours}: clean only the pointwise
      side, keep $\mathcal{P}_1$ intact. SFT $175{,}747$ (pointwise
      $41{,}682 + 42{,}517$, pairwise $44{,}528 + 47{,}020$), RL
      $5{,}945$ (pointwise $3{,}390 + 2{,}555$).
\item $\bigl(\mathcal{S}_1\cup\mathcal{P}_2,\;\mathcal{P}_3\bigr)$
      --- the symmetric counterpart of \textbf{Ours} that cleans only
      the pairwise side. SFT $179{,}280$ (pointwise
      $45{,}072 + 45{,}072$, pairwise $43{,}816 + 45{,}320$), RL
      $2{,}412$ (pairwise $712 + 1{,}700$).
\end{itemize}
We additionally report a sub-row that drops the GRPO stage from
\textbf{Ours},
$\bigl(\mathcal{S}_2\cup\mathcal{P}_1,\;\text{--}\bigr)$, to isolate
the marginal contribution of $\mathcal{S}_3$-RL. All recipes use
Qwen3-VL-4B with $3$ epochs of SFT and (where applicable) $3$ epochs
of GRPO RL.

\section{Supplementary experiments}
\label{app:supp_experiments}

\subsection{Subset-Level Generalization}
\label{app:subset_generalization}
\input{tables/subset_hot_overall}
\paragraph{Subset-level robustness.} Table~\ref{tab:subset-main9-heatmap} visualizes the performance of the nine models configurations across $23$ RoboReward test subsets, using a single global color scale where warmer colors indicate higher scores. \textbf{Ours} exhibits consistently strong performance across heterogeneous robot datasets and is visually close to the proprietary GPT-5-mini teacher on most subsets. In particular, Qwen3-VL-4B with POISE achieves the best open-weight Overall score and remains competitive on both large subsets such as \textsc{robo\_arena} and smaller specialized subsets such as \textsc{berkeley\_mvp}, \textsc{kaist\_nonprehensile}, and \textsc{dlr\_edan}. Compared with raw Qwen3-VL backbones and the single-paradigm RoboReward baselines, the POISE-trained models show a much warmer and more stable pattern across columns, suggesting that the gains are not driven by a single dominant subset but reflect broader cross-dataset generalization. The contrast between POISE and raw-distillation SFT further indicates that multi-paradigm distillation provides strong overall coverage, while our consistency-oriented cleaning improves the reliability of the learned reward signal across diverse embodiment and task distributions.

\input{tables/case_study}

\begin{figure}[htbp]
\centering
\includegraphics[width=\linewidth]{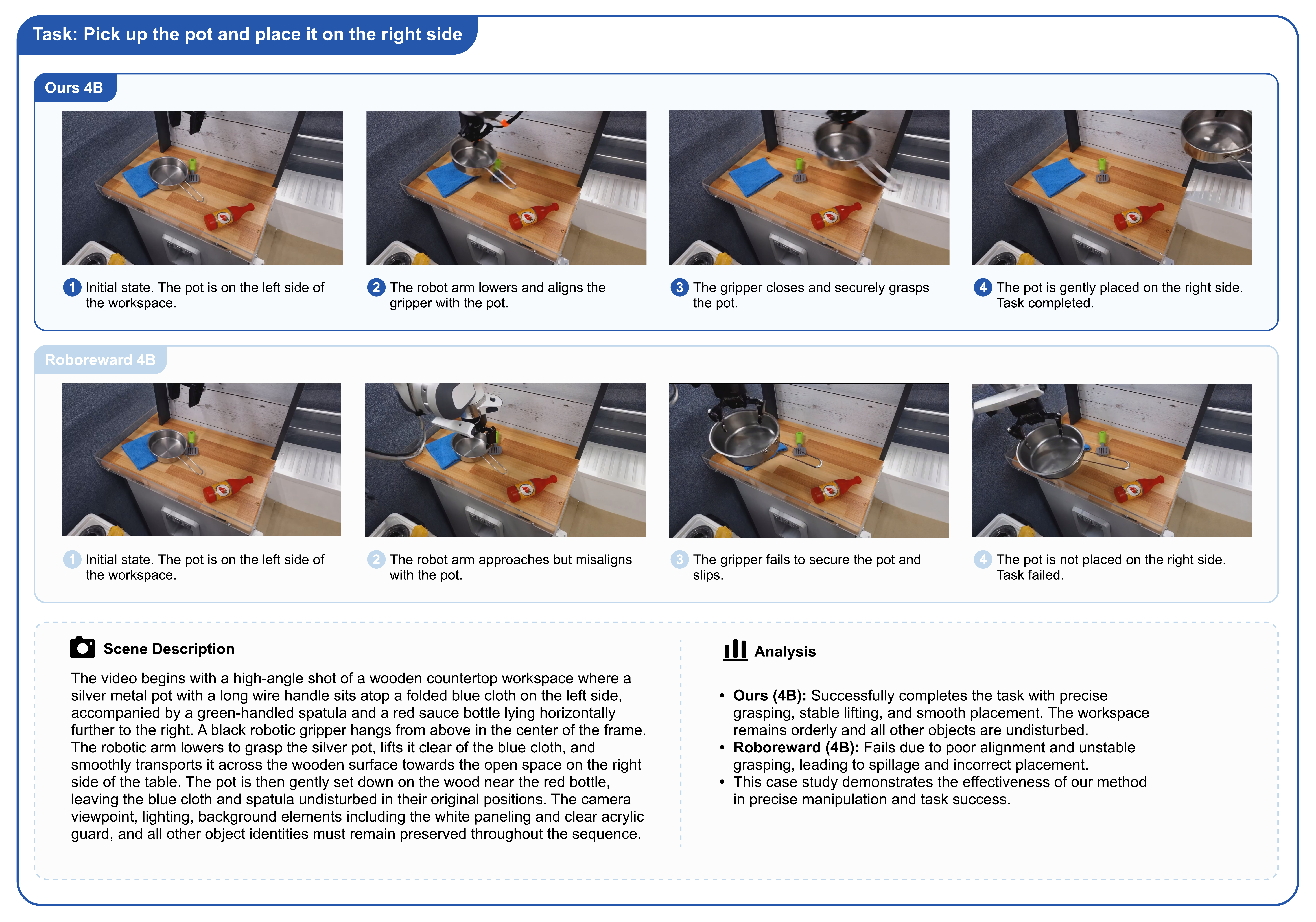}
\caption{Case study comparison between our method and RoboReward on the task of picking up the pot and placing it on the right side. Our method achieves accurate grasping and stable placement, whereas the RoboReward baseline fails due to reward misalignment and unstable manipulation.}
\label{fig:embody_case}
\end{figure}

\subsection{Embodied Case Study}
\label{app:embody_case}

Figure~\ref{fig:embody_case} provides a qualitative comparison on a representative PAIBench-G robot manipulation task. Given the instruction to pick up the pot and place it on the right side, the policy optimized with RoboReward produces visually plausible motion but fails to preserve the task semantics: the grasp is misaligned and the object placement becomes unstable. In contrast, optimization with our reward model yields a trajectory that first establishes a more accurate grasp and then maintains object stability during placement. This example illustrates the same trend observed in Table~\ref{tab:main:core}: our reward model gives a stronger embodied-task signal, especially for semantic correctness and manipulation stability, while preserving comparable visual quality.

\section{Implementation Details}
\label{app:implementation_details}

\subsection{Reward Model and Baseline Training Settings}
\label{app:hparams}

\begin{table}[H]
\caption{Hyperparameters for Supervised Fine-Tuning (SFT) and Group Relative Policy Optimization (GRPO) stages. Both Qwen3-VL-4B and 8B backbones share identical training configurations unless otherwise specified.}
\label{tab:qwen_training_settings}
\centering
\small
\begin{tabularx}{\linewidth}{>{\raggedright\arraybackslash}p{0.32\linewidth}X}
\toprule
\textbf{Hyperparameter} & \textbf{Value} \\
\midrule
\multicolumn{2}{l}{\textbf{General \& Hardware}} \\
\midrule
Backbones & Qwen3-VL-4B-Instruct, Qwen3-VL-8B-Instruct \\
Hardware & 8$\times$ NVIDIA A800 (80 GB) GPUs \\
Trainable parameters & Full-parameter (vision tower frozen; multimodal aligner \& LLM unfrozen) \\
Parallelism \& Memory & FSDP full-shard, gradient checkpointing, FlashAttention-2 \\
Precision & Mixed precision (GRPO Actor uses float16) \\
Optimizer & AdamW, $\beta_1=0.9$, $\beta_2=0.999$, $\epsilon=10^{-6}$ \\
\midrule
\multicolumn{2}{l}{\textbf{Supervised Fine-Tuning (SFT) Stage}} \\
\midrule
Framework & ms-swift\citep{zhao2025swift} \\
Training epochs & 3 (saving checkpoints at each epoch end) \\
Training time (approx.) & $\sim$30 hours (4B) / $\sim$45 hours (8B) \\
Learning rate & $5 \times 10^{-6}$ \\
LR schedule & Cosine \\
Warmup ratio & 0.05 \\
Weight decay & 0.05 \\
Global batch size & 32 (1 sample per device $\times$ 4 grad accumulation $\times$ 8 GPUs) \\
Max sequence length & 32,768 (overlong samples dropped, not truncated) \\
Gradient clipping norm & 1.0 \\
\midrule
\multicolumn{2}{l}{\textbf{Reinforcement Learning (GRPO) Stage}} \\
\midrule
Framework & VeRL\citep{sheng2025hybridflow} \\
Initial policy & Corresponding SFT checkpoint \\
Training epochs & 3 (saving checkpoints at each epoch end) \\
Training time (approx.) & $\sim$20 hours (4B) / $\sim$30 hours (8B) \\
Rollout engine & vLLM (\texttt{gpu\_memory\_utilization=0.75}, tensor parallel size 1) \\
Prompt batch size & 128 \\
Samples per prompt ($n$) & 8 \\
Effective batch size & 1,024 per step \\
Learning rate & $2 \times 10^{-6}$ \\
LR schedule & Constant (no warmup) \\
Max prompt length & 14,336 \\
Max response length & 4,096 \\
PPO clip ratio & 0.2 \\
PPO mini-batch size & 16 \\
PPO epochs per rollout & 1 \\
KL penalty \& Entropy & Disabled (\texttt{use\_kl\_loss=False}, \texttt{kl\_coef=0}, \texttt{entropy\_coeff=0}) \\
Advantage normalization & Intra-group normalization ($\gamma=1$, $\lambda=1$) \\
Gradient clipping norm & 1.0 \\
Additional optimizations & Optimizer state CPU offload (specifically for the 8B backbone) \\
\bottomrule
\end{tabularx}
\end{table}

Table~\ref{tab:qwen_training_settings} lists the full set of hyperparameters and hardware configurations used for instantiating our reward models and baselines. Both the Qwen3-VL-4B and Qwen3-VL-8B backbones share the same hyperparameter setup across the SFT and GRPO training stages, with the only exception being the optimizer state CPU offloading, which is exclusively enabled for the 8B model during GRPO to manage memory constraints.

\paragraph{Compute Resource Usage.} 
All training stages are conducted on a single compute node equipped with 8$\times$ NVIDIA A800 (80 GB) GPUs. For the Qwen3-VL-4B backbone, the SFT stage requires approximately 30 hours, while the GRPO stage takes 20 hours, totaling 400 GPU-hours. For the Qwen3-VL-8B backbone, the SFT and GRPO stages require approximately 45 hours and 30 hours respectively, totaling 600 GPU-hours. Overall, the total compute usage for training the final models reported in this work is approximately 125 node-hours (equivalent to 1,000 A800 GPU-hours).
\subsection{PAIBench-G Robot Reward-Guided Optimization Settings}\label{app:paibench_settings}

\begin{table}[H]
\caption{PAIBench-G robot reward-guided optimization training hyperparameters. Both reward-model variants use the same configuration.}
\label{tab:paibench_settings}
\centering
\small
\begin{tabularx}{\linewidth}{>{\raggedright\arraybackslash}p{0.32\linewidth}X}
\toprule
\textbf{Hyperparameter} & \textbf{Value} \\
\midrule
Backbone & Cosmos-Predict2.5-2B\citep{ali2025world} \\
Algorithm & DiffusionNFT \\
Benchmark & PAIBench-G robot subset \\
Reward & RoboReward / ours \\
Training epochs & 50 \\
Prompt groups & 32 per epoch \\
Group size & 12 \\
Samples per epoch & 384 \\
Batch size & 3 per device \\
Gradient accumulation & 4 \\
Sampling steps & 10 for training, 36 for evaluation \\
Guidance scale & 4.0 for training, 7.0 for evaluation \\
Learning rate & $6\times 10^{-5}$ \\
Optimizer & AdamW, $\beta_1=0.9$, $\beta_2=0.999$, weight decay $10^{-4}$ \\
LoRA & rank 32, alpha 64 \\
Advantage clipping & 2 \\
Training beta & 0.006 \\
Reward beta & 1.0 \\
Precision & bf16 \\
Video format & 93 frames, 432$\times$768, 16 fps \\
Seed & 42 \\
\bottomrule
\end{tabularx}

\end{table}

Table~\ref{tab:paibench_settings} lists the full set of hyperparameters used in the PAIBench-G robot reward-guided optimization experiment. The RoboReward baseline and our method share an identical configuration, differing only in the reward model.

\section{Prompts}
\label{app:prompts}

This appendix reproduces, verbatim, the prompts used throughout our pipeline: at distillation time (querying the GPT-5-mini teacher to build the four-paradigm corpus of §\ref{sec:method-mpd}), at student training time, and at inference/evaluation time. Keeping a single set of prompt templates across training and inference eliminates a common distribution-shift source and ensures that teacher outputs, student training labels, and benchmark evaluations are all directly comparable. All four paradigms share a single system prompt and differ only in the user template. The placeholder \texttt{<video>} marks where the model receives the video frames; \texttt{\{action\_text\}}, \texttt{\{question\}}, \texttt{\{answer\_text\}}, \texttt{\{answer\_a\}}, and \texttt{\{answer\_b\}} are filled in at runtime.

\subsection{Shared System Prompt}
\label{app:prompt_system}

The same system message is prepended to all four paradigms. It pins the teacher to a strict, video-grounded judge persona and enforces the required XML-style answer envelope that our downstream parsers in \texttt{verl\_common.py} (\texttt{parse\_score\_from\_output}, \texttt{parse\_pair\_label\_from\_output}) expect.

\begin{promptbox}{Prompt: Shared System}
You are a strict, impartial, and consistent video judge for robot task
evaluation.
You must evaluate robot performance based solely on what is visible in
the video.
Be concrete and factual. Do not invent objects or events that are not
visible.
Follow the provided rubric exactly and output ONLY in the required XML
format.
\end{promptbox}

\subsection{Score-A: Trajectory Progress Scoring}
\label{app:prompt_scorea}

Score-A asks the teacher to rate a single trajectory's task progress on a $1$-to-$5$ rubric, matching RoboReward's original scoring protocol. The rubric explicitly operationalizes the 3-vs-4 boundary (partial vs.\ near completion) in terms of primary object identity and primary spatial relation, which is what makes Score-A labels consistent enough with Pair-A labels for POISE's isotonic correction to be well-defined.

\begin{promptbox}{Prompt: Score-A User Template}
<video>
Task Instruction:
{action_text}

Scoring Rubric (judge only the final state, ignore time limits):
1 - No Success: The final state shows no goal-relevant change for the
task (e.g., the robot handled a completely different object or made no
meaningful motion).
2 - Minimal Progress: The final state shows a small but insufficient
change toward the goal (e.g., the robot moved toward the target but
did not meaningfully interact with it).
3 - Partial Completion: The final state shows good progress toward the
goal but violates more than one requirement or a major requirement
(e.g., the primary spatial relation is incorrect, or the wrong object
identity was used).
4 - Near Completion: The final state is correct in region and intent
but misses a single minor requirement (e.g., the target object is
correctly placed relative to the reference, but an auxiliary
constraint like gripper release, orientation tolerance, or a secondary
support surface is missed).
5 - Perfect Completion: The final state fully satisfies all the
requirements of the task.

Clarification for scoring 3 vs 4 (MINOR vs MAJOR requirements):
- Treat a requirement as MINOR when the primary object identity and
primary spatial relation to the reference object are satisfied, but
an auxiliary constraint is missed.
- Examples of auxiliary (MINOR) constraints: remaining on the same
support surface, holding/releasing the gripper at the end, small
orientation or placement tolerances, or cosmetic positioning details
that do not change the core relation.
- Use score 4 (Near Completion) if all core elements are correct but
one auxiliary constraint is missed.
- Use score 3 (Partial Completion) only when multiple constraints are
missed or when the missed constraint is central to the task identity.

Important:
- Judge ONLY the final state shown in the video, not intermediate
steps.
- Be concrete and factual. Do not invent objects or events not visible
in the video.
- Pay special attention to the last frame for the final state
assessment.

Return an integer score from 1 to 5 and a brief reason (1-2 sentences).
Output format:
<reason>...</reason>
<answer>Score: [X]</answer>
\end{promptbox}

\subsection{Score-B: Video-QA Answer Quality Scoring}
\label{app:prompt_scoreb}

Score-B rates a candidate answer to a grounding question on the same $1$-to-$5$ scale, forcing the judge to verify objects, actions, and states rather than rely on surface textual cues. Hallucination is explicitly named as the primary failure mode in the rubric.

\begin{promptbox}{Prompt: Score-B User Template}
<video>
Question:
{question}

Candidate Answer:
{answer_text}

Answer Quality Rubric:
1 - Completely incorrect: The answer is entirely wrong, hallucinated,
or not grounded in the video content. It contradicts what is visible
or fabricates objects/events that do not appear.
2 - Slightly relevant: The answer shows some awareness of the video
topic but is mostly wrong, missing key details, or includes
significant hallucinations.
3 - Partially correct: The answer captures some correct elements but
has notable errors, misidentifies important objects or actions, or
omits important points that are clearly visible in the video.
4 - Mostly correct: The answer is largely accurate and grounded in the
video, with only minor omissions, imprecise descriptions, or trivial
inaccuracies.
5 - Completely correct: The answer is fully accurate, well-grounded in
the video, concise, and covers all key objects and actions relevant
to the question.

Important:
- Evaluate the answer strictly based on what is visible in the video.
- Penalize hallucinated content: objects, actions, or details not
present in the video.
- Consider factual accuracy, visual grounding, and coverage of key
details.

Return an integer score from 1 to 5 and a brief reason (1-2 sentences).
Output format:
<reason>...</reason>
<answer>Score: [X]</answer>
\end{promptbox}

\subsection{Pair-A: Dual-Trajectory Preference}
\label{app:prompt_paira}

Pair-A compares two robot execution videos on the same task and emits a label in $\{\texttt{[A]},\texttt{[B]},\texttt{[C]}\}$ (A wins / B wins / tie), providing the pairwise format that DPO and Bradley-Terry pipelines consume natively. Two design choices in this prompt are relevant to the main paper:
\begin{itemize}
\item The prompt instructs the teacher to analyze \emph{Candidate B first, then Candidate A}, which is our cheap position-bias mitigation at distillation time and complements the symmetric \texttt{order\_ab}/\texttt{order\_ba} aggregation in §\ref{sec:method-mpd}.
\item Tie (\texttt{[C]}) is defined restrictively---``both videos genuinely show the same level of task completion''---so that Pair-A labels remain informative partial-order constraints for POISE's isotonic correction (§\ref{sec:method-poise}) rather than collapsing into uninformative ties.
\end{itemize}

\begin{promptbox}{Prompt: Pair-A User Template}
=== Video A (frames below) ===
<video>
=== Video B (frames below) ===
<video>
The FIRST video corresponds to Candidate A.
The SECOND video corresponds to Candidate B.

Task Instruction:
{action_text}

You are comparing two robot execution videos for the same task.
Choose the video that shows better task completion. Your evaluation
should consider:
- Whether the correct object was manipulated
- Whether the target spatial relation was achieved
- Whether all task requirements were satisfied
- The overall degree of progress toward the goal

The final outcome (what state the task is in at the end) is the
primary criterion. The execution process (efficiency, smoothness,
side effects) serves as a secondary tiebreaker when both candidates
reach similar final states.

IMPORTANT: To avoid order bias, analyze Candidate B FIRST, then
Candidate A.

Step 1: Watch Candidate B fully. Describe what the robot did and what
the final state looks like.
Step 2: Watch Candidate A fully. Describe what the robot did and what
the final state looks like.
Step 3: Compare the two. Which video shows more progress toward
completing the task? If one clearly achieved more, choose it.
Only declare [C] (tie) when both videos genuinely show the same level
of task completion and you cannot identify any meaningful difference.

Be concrete and factual - do not invent objects or events not visible
in the videos. Avoid any position biases and ensure that the order in
which the videos were presented does not influence your judgment.

Output format:
<analysis_B>Candidate B: what the robot did and the final state
</analysis_B>
<analysis_A>Candidate A: what the robot did and the final state
</analysis_A>
<reason>Brief comparison - which achieved more toward the task goal
</reason>
<answer>[A]</answer> or <answer>[B]</answer> or <answer>[C]</answer>
\end{promptbox}

\subsection{Pair-B: Dual-Answer Preference}
\label{app:prompt_pairb}

Pair-B compares two candidate answers to the same grounding question with the same $\{\texttt{[A]},\texttt{[B]},\texttt{[C]}\}$ label vocabulary as Pair-A. The rubric elevates factual accuracy and absence of hallucination over coverage or length so that Pair-B agrees, on jointly-labeled items, with Score-B's hallucination-penalizing rubric in §\ref{app:prompt_scoreb}. The same reverse-order instruction as in Pair-A is applied to mitigate order bias.

\begin{promptbox}{Prompt: Pair-B User Template}
<video>
Question:
{question}

Candidate A Answer:
{answer_a}

Candidate B Answer:
{answer_b}

You are comparing two candidate answers to the above question about
the video.
Choose the answer that better addresses the question. Your evaluation
should consider:
- Factual accuracy: does the answer correctly describe what happened
in the video?
- Absence of hallucination: does the answer avoid fabricating objects,
actions, or details not visible?
- Visual grounding: is the answer based on what is actually shown in
the video?
- Coverage: does the answer address the key objects and actions
relevant to the question?

Factual accuracy and absence of hallucination are the primary
criteria. Coverage and level of detail serve as secondary tiebreakers
when both candidates are equally accurate.
Do not let the length of the answers influence your evaluation.
Evaluate strictly based on what is visible in the video.

IMPORTANT: To avoid order bias, analyze Candidate B FIRST, then
Candidate A.

Step 1: Read Candidate B. Assess its factual accuracy and whether it hallucinates.
Step 2: Read Candidate A. Assess its factual accuracy and whether it
hallucinates.
Step 3: Compare the two. Which answer is more accurate and better
grounded in the video? If one is clearly more accurate, choose it.
Only declare [C] (tie) when both answers are genuinely equal in
quality and you cannot identify any meaningful difference.

Be concrete and factual - do not invent details not visible in the
video. Avoid any position biases and ensure that the order in which
the answers were presented does not influence your judgment.

Output format:
<analysis_B>Candidate B: accuracy assessment and key observations
</analysis_B>
<analysis_A>Candidate A: accuracy assessment and key observations
</analysis_A>
<reason>Brief comparison - which answer is more accurate and complete
</reason>
<answer>[A]</answer> or <answer>[B]</answer> or <answer>[C]</answer>
\end{promptbox}


\newpage

\end{document}

%% file: exp.tex
\section{Experiments}
\label{sec:exp}

We evaluate our approach in five stages. First, we detail the experimental setup including datasets, baselines, and evaluation metrics (\S\ref{sec:exp:preliminaries}). Next, we present main results comparing our models against proprietary and open-source baselines (\S\ref{sec:exp:experiments}). We then conduct ablations on data-routing, $\mathcal{S}_3$-RL contribution, and TrustJudge realignment mechanisms. Fourth, we validate our reward models through a human correlation study. Finally, we demonstrate our reward model's superiority in real-world embodied AI scenarios, showing it provides winning guidance for robot policy optimization. Additionally, supplementary experiments are provided in Appendix~\ref{app:supp_experiments}, including a detailed performance heatmap across all test sub-datasets in Appendix~\ref{app:subset_generalization}.
\subsection{Preliminaries}
\label{sec:exp:preliminaries}
\textbf{Datasets and Splits.}
We conduct our experiments using a total of $181{,}692$ distillation training samples and a held-out gold set of $7{,}030$ human-annotated test samples. The construction of specific tasks (Score-A/B and Pair-A/B) from the RoboReward dataset is detailed in \S\ref{sec:method}. Following POISE pipeline as described in \S\ref{sec:method}, the raw pointwise ($\mathcal{S}_1$) and pairwise ($\mathcal{P}_1$) samples are processed into cleaned SFT-admissible subsets ($\mathcal{S}_2, \mathcal{P}_2$) and RL-only residual subsets ($\mathcal{S}_3, \mathcal{P}_3$). For detailed statistical distributions of all training and evaluation splits, please refer to Appendix \ref{app:stats}.

\textbf{Models and Baselines.}
We instantiate our reward models using \textbf{Qwen3-VL-4B} and \textbf{Qwen3-VL-8B} backbones. We evaluate against three baselines: 
(i) \textbf{Raw-distillation SFT} $(\mathcal{S}_1\cup\mathcal{P}_1,\,\text{--})$, representing the same backbones trained on uncleaned multi-paradigm distillation data; 
(ii) \textbf{RoboReward}~\citep{lee2026roboreward}, the dominant prior recipe; and 
(iii) \textbf{GPT-5-mini}, which serves as a soft upper bound for the recoverable performance under our distillation setup.
We also include the un-tuned Qwen3-VL backbones as a reference floor to quantify the absolute headroom. Implementation details are provided in Appendix~\ref{app:implementation_details}, including detailed hyperparameters, training specifics, and compute resource usage in Appendix~\ref{app:hparams}.

\textbf{Metrics.}
We index the four sub-tasks by $t\in\{\textsc{Score-A},\textsc{Score-B}, \textsc{Pair-A},\textsc{Pair-B}\}$, where \textsc{A} denotes the trajectory-progress modality and \textsc{B} denotes the video-QA modality. The performance is evaluated using the following metrics:

\begin{itemize}[leftmargin=*]
    \item \textbf{Pointwise Scores (\textsc{Score-A} and \textsc{Score-B}):} 
    Let $\mathcal{D}_t$ be the pointwise test set for $t\in\{\textsc{Score-A},\textsc{Score-B}\}$. For each example $(x_i,s_i) \in \mathcal{D}_t$, $x_i$ denotes the input context, $s_i \in \{1,\dots,5\}$ is the integer ground-truth rating, and $\hat s_i \in \{1,\dots,5\}$ represents the predicted rating from the reward model. We normalize by the maximal possible error on the 1 to 5 scale (which is 4) and scale the results by 100 so that the scores lie in $[0,100]$: {\boldmath $\displaystyle \textbf{\textsc{Score}}\text{-}t = 100 \times \left( 1 \;-\; \frac{1}{4\,\lvert\mathcal{D}_t\rvert} \sum_{(x_i,s_i)\in\mathcal{D}_t} \bigl\lvert\hat s_i - s_i\bigr\rvert \right), \quad t \in \{\textsc{Score-A}, \textsc{Score-B}\}$.}

    \item \textbf{Pairwise Accuracies (\textsc{Pair-A} and \textsc{Pair-B}):}
Let $\mathcal{D}_t$ be the pairwise test set for
$t\in\{\textsc{Pair-A},\textsc{Pair-B}\}$, where each example consists of
an input $x_i$, two candidate responses, and a dataset label $z_i\in\{\mathrm{win},\mathrm{lose},\mathrm{tie}\}$
indicating whether the first response wins, loses, or ties against the second response. Let $\hat z_i$ be the label predicted by the
reward model. The accuracy is calculated as:
{\boldmath $\displaystyle
\textbf{\textsc{Pair}}\text{-}t
=
100 \times
\frac{1}{\lvert\mathcal{D}_t\rvert}
\sum_{i=1}^{\lvert\mathcal{D}_t\rvert}
\mathbf{1}\!\bigl[\hat z_i=z_i\bigr],
\quad
t\in\{\textsc{Pair-A},\textsc{Pair-B}\}.
$}

    \item \textbf{Paradigm Aggregates ($\textsc{Score}_{\mathrm{AB}}$ and $\textsc{Pair}_{\mathrm{AB}}$):} 
    The two paradigm aggregates are formed by sample-weighted (micro) averaging. We calculate this by summing the sub-task numerators and denominators rather than averaging the sub-task ratios. This ensures that the smaller of the two test sets is not over-weighted when the modalities are unbalanced. Let $M \in \{\textsc{Score}, \textsc{Pair}\}$ denote the metric type, and let $\mathcal{D}_t$ be the respective test set for sub-task $t$. The aggregated metric $M_{\mathrm{AB}}$ is computed as: {\boldmath $\displaystyle M_{\mathbf{AB}} = \frac{ \lvert\mathcal{D}_{M\text{-A}}\rvert \cdot M\text{-A} + \lvert\mathcal{D}_{M\text{-B}}\rvert \cdot M\text{-B} }{ \lvert\mathcal{D}_{M\text{-A}}\rvert + \lvert\mathcal{D}_{M\text{-B}}\rvert }, \quad M \in \{\textsc{Score}, \textsc{Pair}\}$.}

    \item \textbf{Overall Score:} 
    The \textbf{Overall} score is the equally weighted average of the micro-averaged pointwise score ($\textsc{Score}_{\mathrm{AB}}$) and the micro-averaged pairwise accuracy ($\textsc{Pair}_{\mathrm{AB}}$) over both modalities. It is calculated as: {\boldmath $\displaystyle \textbf{\textsc{Overall}} = (\textbf{\textsc{Score}}_{\mathbf{AB}} + \textbf{\textsc{Pair}}_{\mathbf{AB}}) / 2$.}

    \item \textbf{Cross-Paradigm Self-Consistency ($\mathrm{Cons.}$):} 
    This metric measures whether the ordering implied by the pointwise scores agrees with the direct pairwise judgments on the same comparable pairs. Following TrustJudge, we calculate the score-comparison inconsistency rate, also known as the Conflict Ratio ($\mathrm{CR}$). Let $\mathrm{CR}_{\textsc{A}}$ and $\mathrm{CR}_{\textsc{B}}$ denote the inconsistency rates on \textsc{A} and \textsc{B}, respectively. We define the overall consistency as one minus their average: {\boldmath $\mathrm{Cons.} = 100 \times \left( 1 - \frac{\mathrm{CR}_{\textsc{A}} + \mathrm{CR}_{\textsc{B}}}{2} \right)$.}
\end{itemize}
\subsection{Experiments}
\label{sec:exp:experiments}
\input{tables/main_core.tex}
\textbf{Main comparison.}
Table~\ref{tab:main:core} reports the head-to-head comparison of our method against baselines on the Qwen3-VL-4B and Qwen3-VL-8B architectures. Overall, \textbf{Ours 4B} strictly matches the performance of the proprietary teacher, GPT-5-mini, within a $0.13$ points margin. Furthermore, it outperforms every open-weight baseline at both scales and delivers the most substantial gains in cross-paradigm self-consistency.
The key findings are as follows:

\begin{itemize}[leftmargin=*]
\item \textbf{Approaching the performance of proprietary teachers.} Remarkably, Qwen3-VL-4B trained with our methodology reaches an Overall score of $77.96$, trailing the massive GPT-5-mini teacher ($78.09$) by a mere $0.13$ points. It even surpasses the teacher on the Pair-A sub-task ($83.18$ vs.\ $82.44$). This demonstrates that multi-paradigm distillation establishes a robust, highly competitive foundation for absolute task accuracy.

In summary, multi-paradigm distillation lifts the 4B open-weight student to teacher-level accuracy and narrows the observed 4B-8B gap under our protocol. Human-annotated cycle-breaking removes structural conflicts in distillation data, suggesting reliable vision-language reward modeling depends more on supervision coverage and targeted data correction than on parameter scaling alone.

\item \textbf{Driving self-consistency via POISE after human-annotated cycle-breaking.} Naively concatenating pointwise ($\mathcal{S}_1$) and pairwise ($\mathcal{P}_1$) datasets introduces conflicting reward signals, which causes outputs to disagree on roughly one-third of all evaluated pairs. By leveraging targeted human annotation to break these logical cycles, we successfully reduce this residual disagreement to approximately one-quarter. This mechanism drives massive improvements in cross-paradigm self-consistency ($\mathrm{Cons.}$), where scores jump from $56.14 \to 63.66 \to 73.58$ on the $8$B backbone (a $+17.44$ total gain) and from $57.26 \to 64.81 \to 71.90$ on the $4$B backbone, all while maintaining robust performance without compromising absolute task accuracy.
\end{itemize}


In summary, multi-paradigm distillation lifts 4B open-weight student to teacher-level accuracy and narrows the performance gap between the 4B and 8B models measured under our evaluation protocol. Human-annotated cycle-breaking further removes structural conflicts in the distillation data, suggesting that reliable vision-language reward modeling depends more on broad supervision coverage and targeted data correction than on parameter scaling alone.


\input{tables/main_setting4b.tex}

\textbf{Data-routing ablation.} 
Table~\ref{tab:main:4b} evaluates four (SFT, RL) data-routing recipes on Qwen3-VL-4B; their explicit tuple definitions are given in Appendix~\ref{app:stats}. Our optimal configuration, \textbf{Ours 4B} achieves a top score of $77.96$. This performance surpasses the fully-cleaned ($77.53$), raw distillation ($77.09$), and pairwise-only-cleaned ($75.48$) alternatives. Notably, substituting $\mathcal{P}_1$ with the refined $\mathcal{P}_2$ in the SFT stage slightly degrades performance. This suggests that with pointwise data cleaned to $\mathcal{S}_2$, the self-consistent anchors effectively mitigate residual cycles in $\mathcal{P}_1$, allowing its superior diversity to outweigh the cycle costs. We identify a critical asymmetry: while cleaning the pointwise side ($\mathcal{S}_1\to\mathcal{S}_2$) is essential, cleaning the pairwise side ($\mathcal{P}_1\to\mathcal{P}_2$) imposes a sub-optimal trade-off between diversity and consistency.

\textbf{$\mathcal{S}_3$-RL contribution.} 
Ablating the GRPO stage (\emph{Ours w/o RL}) reveals that $\mathcal{S}_3$-RL serves as a targeted lever for score--pair alignment instead of a general accuracy booster. While the RL stage yields a modest $1.37$ point gain in Overall accuracy ($76.59 \to 77.96$), it triggers a substantial $+4.8$ point surge in self-consistency ($\mathrm{Cons.}: 67.08 \to 71.90$). This disproportionate impact is approximately $3.5\times$ greater on consistency than on accuracy, which validates the design of $\mathcal{S}_3$.

\textbf{POISE is more effective than training-time TrustJudge.} Table~\ref{tab:ablation:tj} shows a clear stage difference. Training-time TrustJudge lowers \textbf{Ours 4B} from $77.96$ to $77.02$ Overall, mainly through pointwise drops (Score-A $-3.0$, Score-B $-0.9$). This suggests that training-time TrustJudge may be less effective than POISE, as it samples multiple TrustJudge judgments for each training example, aggregates them into a single corrected label, and uses that label as the SFT target, which may reduce label diversity and weaken the benefits of probability-based aggregation. In contrast, inference-time TrustJudge raises \textbf{Ours 4B} to $78.57$ Overall, achieves the best Pair-A result ($85.27$), and also improves RoboReward-4B and raw Qwen3-VL-4B. Its slight Cons. drop after RL/GRPO likely comes from reduced output entropy: a more confident model leaves less useful uncertainty for probability-based aggregation, making TrustJudge less effective and slightly weakening score--pair self-consistency.

\begin{wrapfigure}{r}{0.5\textwidth}

\centering
\includegraphics[width=0.5\textwidth]{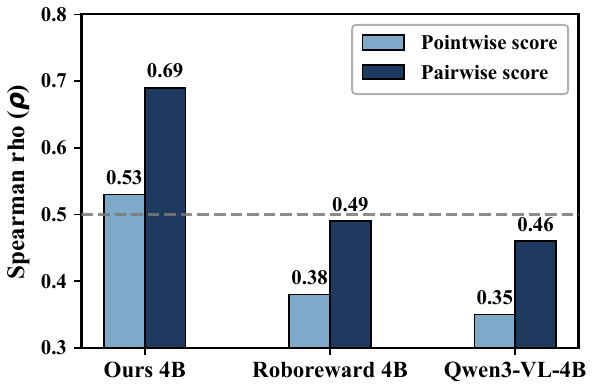}

\caption{Spearman $\rho$ between reward-model judgments and human
annotations, broken down by output format. Dashed line: conventional
strong-correlation threshold $\rho{=}0.50$. Higher is better.}
\label{fig:human-corr}

\end{wrapfigure}
\textbf{Human-annotation correlation.}


\input{tables/ablation_trustjudge.tex}


To evaluate reward-model reliability beyond automatic metrics, we sample a subset of $1{,}000$ random test instances with fine-grained human preference ratings to report the Spearman correlation $\rho$ with human preferences. Figure~\ref{fig:human-corr} shows that \textbf{Ours 4B} achieves the highest correlation in both formats ($\rho=0.42$ pointwise, $\rho=0.55$ pairwise) and is the only $4$B-scale model exceeding the $\rho=0.50$ strong-correlation threshold. Compared to the RoboReward baseline ($0.38, 0.49$), \textbf{Ours 4B} improves by $+0.04$ and $+0.06$, respectively. Across all models, pairwise outputs correlate more strongly with humans than pointwise ones, matching prior findings that pairwise judgments yield higher agreement in Table~\ref{tab:main:core}. This larger pairwise gain suggests our cleaning pipeline strengthens the model precisely where traditional single-paradigm supervision is most vulnerable.

\textbf{Reward-Guided Optimization on PAIBench-G Robot Tasks.}

We evaluate reward-guided policy optimization on a sampled subset of PAIBench-G robot tasks. The task is image-to-video prediction from an initial robot observation and a manipulation instruction. We conduct a head-to-head comparison between our 4B reward model and RoboReward-4B over 68 paired comparisons, judged by 3 annotators with majority vote. Our method achieves approximately 69\% win rate against RoboReward-4B, demonstrating that our reward model provides a better optimization signal for embodied task completion. More parameter details and a qualitative example are provided in Appendix~\ref{app:paibench_settings} and Figure~\ref{fig:embody_case}.

%% file: tables/main_core.tex
\begin{table}[t]
\centering
\caption{Main comparison on the four sub-tasks. ``--'' denotes that the corresponding stage is
skipped, and all trained rows use $3$ epochs for each stage that is
run. }
\label{tab:main:core}
\setlength{\tabcolsep}{3.4pt}
\small
\resizebox{\textwidth}{!}{%
\begin{tabular}{llllccccccccc}
\toprule
\textbf{Model} & \textbf{Recipe} & \textbf{SFT data} & \textbf{RL data} & \textbf{Overall} & \textbf{scoreAB} & \textbf{pairAB} & \textbf{Score-A} & \textbf{Score-B} & \textbf{Pair-A} & \textbf{Pair-B} & \textbf{Cons.} \\
\midrule
\multicolumn{12}{l}{\emph{Proprietary teacher (upper-bound reference)}} \\
GPT-5-mini   & teacher                        & --                                         & --                                         & 78.09 & 74.52 & 81.65 & 75.05 & 73.98 & 82.44 & 80.89 & 68.09 \\
\midrule
\multicolumn{12}{l}{\emph{Ours}} \\
Qwen3-VL-4B     & Ours                           & $\mathcal{S}_2 \cup \mathcal{P}_1$         & $\mathcal{S}_3$                            & \textbf{77.96} & \underline{74.49} & \textbf{81.43} & \underline{75.98} & \underline{73.00} & \textbf{83.18} & \underline{79.74} & \underline{71.90} \\
Qwen3-VL-8B     & Ours                           & $\mathcal{S}_2 \cup \mathcal{P}_1$         & $\mathcal{S}_3$                            & \underline{77.46} & \textbf{74.51} & \underline{80.41} & \textbf{75.59} & \textbf{73.44} & \underline{80.80} & \textbf{80.03} & \textbf{73.58} \\
\midrule
\multicolumn{12}{l}{\emph{Baselines}} \\
Qwen3-VL-8B     & raw-distillation SFT           & $\mathcal{S}_1 \cup \mathcal{P}_1$         & --                                         & 77.21 & 73.79 & 80.63 & 73.70 & 73.87 & 83.04 & 78.30 & 63.66 \\
Qwen3-VL-4B     & raw-distillation SFT           & $\mathcal{S}_1 \cup \mathcal{P}_1$         & --                                         & 77.09 & 73.18 & 80.99 & 73.15 & 73.20 & 82.44 & 79.60 & 64.81 \\
Qwen3-VL-8B     & RoboReward                     & --                            & --                                         & 69.82 & 74.28 & 65.35 & 77.95 & 70.61 & 60.86 & 69.68 & 56.14 \\
Qwen3-VL-4B     & RoboReward                     & --                            & --                                         & 67.83 & 72.58 & 63.08 & 74.13 & 71.04 & 58.33 & 67.67 & 57.26 \\
\midrule
\multicolumn{12}{l}{\emph{Reference floor (no fine-tune)}} \\
Qwen3-VL-4B     & --                             & --                                         & --                                         & 62.06 & 66.31 & 57.82 & 65.96 & 66.65 & 55.80 & 59.77 & 62.20 \\
Qwen3-VL-8B     & --                             & --                                         & --                                         & 60.47 & 65.47 & 55.48 & 71.80 & 59.13 & 52.83 & 58.05 & 54.87 \\
\bottomrule
\end{tabular}}
\end{table}

%% file: tables/main_setting4b.tex
\begin{table}[t]
\centering
\caption{Effect of the SFT and RL data routing on Qwen3-VL-4B. ``--'' denotes that the corresponding stage is skipped. Sub-row drops the GRPO stage from \textbf{Ours} to isolate the
contribution of $\mathcal{S}_3$-RL. All rows use $3$ epochs for each stage that is run. }
\label{tab:main:4b}
\setlength{\tabcolsep}{3.6pt}
\small
\resizebox{\textwidth}{!}{%
\begin{tabular}{lllccccccccc}
\toprule
\textbf{Recipe} & \textbf{SFT data} & \textbf{RL data} & & \textbf{Overall} & \textbf{scoreAB} & \textbf{pairAB} & \textbf{Score-A} & \textbf{Score-B} & \textbf{Pair-A} & \textbf{Pair-B} & \textbf{Cons.} \\
\midrule
Raw distillation        & $\mathcal{S}_1 \cup \mathcal{P}_1$ & --                                 & & 77.09          & 73.18          & \underline{80.99} & 73.15             & 73.20          & \underline{82.44} & \underline{79.60} & 64.81 \\
Fully-cleaned           & $\mathcal{S}_2 \cup \mathcal{P}_2$ & $\mathcal{S}_3 \cup \mathcal{P}_3$ & & \underline{77.53} & \textbf{75.02} & 80.04          & \textbf{76.04}    & \textbf{73.99} & 81.10          & 79.02          & \underline{69.79} \\
Pairwise-only-cleaned   & $\mathcal{S}_1 \cup \mathcal{P}_2$ & $\mathcal{P}_3$                    & & 75.48          & 72.17          & 78.80          & 72.07             & 72.26          & 79.76          & 77.87          & 64.94 \\
\emph{Ours w/o RL}      & $\mathcal{S}_2 \cup \mathcal{P}_1$ & --                                 & & 76.59          & 72.49          & 80.70          & 73.04             & 71.94          & 82.89          & 78.59          & 67.08 \\
\midrule
\textbf{Ours}           & $\mathcal{S}_2 \cup \mathcal{P}_1$ & $\mathcal{S}_3$                    & & \textbf{77.96} & 74.49          & \textbf{81.43} & \underline{75.98} & 73.00          & \textbf{83.18} & \textbf{79.74} & \textbf{71.90} \\
\bottomrule
\end{tabular}}
\end{table}

%% file: tables/ablation_trustjudge.tex
\begin{table}[t]
\centering
\caption{TrustJudge ablation on Qwen3-VL-4B. We compare applying TrustJudge alignment to the training-time labels, to the inference-time outputs, or not at all (``--'').}
\label{tab:ablation:tj}
\setlength{\tabcolsep}{3.4pt}
\small
\resizebox{\textwidth}{!}{%
\begin{tabular}{llcccccccc}
\toprule
\textbf{Reward model} & \textbf{TrustJudge} & \textbf{Overall} & \textbf{scoreAB} & \textbf{pairAB} & \textbf{Score-A} & \textbf{Score-B} & \textbf{Pair-A} & \textbf{Pair-B} & \textbf{Cons.} \\
\midrule
Ours              & inference & \textbf{78.57} & \textbf{74.54} & \textbf{82.60} & \textbf{76.08} & \textbf{73.00} & \textbf{85.27} & \textbf{80.03} & \underline{68.81} \\
Ours              & --        & \underline{77.96}          & \underline{74.49}          & 81.43          & \underline{75.98}          &\textbf{ 73.00} & 83.18          & \underline{79.74}          & \textbf{71.90} \\
Ours              & training  & 77.02          & 72.54          & \underline{81.51}          & 72.94          & 72.14 & \underline{83.33}          & 79.74          & 62.02 \\
\midrule
RoboReward        & inference & 70.29          & 73.03          & 67.54          & 74.95          & 71.11 & 61.16          & 73.71          & 65.55 \\
Qwen3-VL-4B (raw) & inference & 64.56          & 66.11          & 63.01          & 65.69          & 66.52 & 61.31          & 64.66          & 60.46 \\
\bottomrule
\end{tabular}}
\end{table}

%% file: tables/subset_hot_overall.tex
\begin{table*}[!htbp]
\centering
\caption{\textbf{Subset-level generalization across RoboReward test subsets.} (a) maps each subset ID to its corresponding RoboReward test subset and sample count. (b) reports Overall and per-subset scores for the nine main-table configurations. Colors in (b) are normalized globally over all displayed values, with warmer colors indicating higher performance.}
\label{tab:subset-main9-heatmap}
\scriptsize
\setlength{\tabcolsep}{5.0pt}
\renewcommand{\arraystretch}{1.08}

\textbf{(a) Subset index.}

\begin{adjustbox}{width=\textwidth}
\begin{tabular}{@{}c l r c l r@{}}
\toprule
Column & Subset & $n$ & Column & Subset & $n$ \\
\midrule
S01 & \texttt{robo\_arena} & 2000 & S13 & \texttt{cmu\_play\_fusion} & 236 \\
S02 & \texttt{fractal20220817\_data} & 356 & S14 & \texttt{austin\_sirius\_dataset\_converted\_externally\_to\_rlds} & 228 \\
S03 & \texttt{ucsd\_kitchen\_dataset\_converted\_externally\_to\_rlds} & 326 & S15 & \texttt{taco\_play} & 228 \\
S04 & \texttt{bridge} & 292 & S16 & \texttt{berkeley\_rpt\_converted\_externally\_to\_rlds} & 222 \\
S05 & \texttt{berkeley\_autolab\_ur5} & 250 & S17 & \texttt{iamlab\_cmu\_pickup\_insert\_converted\_externally\_to\_rlds} & 218 \\
S06 & \texttt{jaco\_play} & 250 & S18 & \texttt{utokyo\_xarm\_bimanual\_converted\_externally\_to\_rlds} & 202 \\
S07 & \texttt{roboturk} & 250 & S19 & \texttt{utokyo\_pr2\_tabletop\_manipulation\_converted\_externally\_to\_rlds} & 182 \\
S08 & \texttt{berkeley\_fanuc\_manipulation} & 248 & S20 & \texttt{berkeley\_mvp\_converted\_externally\_to\_rlds} & 178 \\
S09 & \texttt{droid} & 248 & S21 & \texttt{viola} & 156 \\
S10 & \texttt{tokyo\_u\_lsmo\_converted\_externally\_to\_rlds} & 248 & S22 & \texttt{kaist\_nonprehensile\_converted\_externally\_to\_rlds} & 142 \\
S11 & \texttt{ucsd\_pick\_and\_place\_dataset\_converted\_externally\_to\_rlds} & 248 & S23 & \texttt{dlr\_edan\_shared\_control\_converted\_externally\_to\_rlds} & 80 \\
S12 & \texttt{stanford\_hydra\_dataset\_converted\_externally\_to\_rlds} & 242 & & & \\
\bottomrule
\end{tabular}
\end{adjustbox}

\setlength{\tabcolsep}{2.0pt}

\textbf{(b) Subset-level performance heatmap.}

\begin{adjustbox}{width=\textwidth}
\begin{tabular}{@{}l l c c *{24}{r}@{}}
\toprule
Model & Recipe & SFT data & RL data & Overall & S01 & S02 & S03 & S04 & S05 & S06 & S07 & S08 & S09 & S10 & S11 & S12 & S13 & S14 & S15 & S16 & S17 & S18 & S19 & S20 & S21 & S22 & S23 \\
\midrule
\multicolumn{28}{l}{\textit{Proprietary teacher (upper-bound reference)}} \\
GPT-5-mini & teacher & -- & -- & \cellcolor[HTML]{F6B988}0.781 & \cellcolor[HTML]{F9CF9B}0.724 & \cellcolor[HTML]{F4AD7F}0.810 & \cellcolor[HTML]{F4AA7C}0.819 & \cellcolor[HTML]{F0946A}0.874 & \cellcolor[HTML]{F19A6F}0.858 & \cellcolor[HTML]{F8CB97}0.735 & \cellcolor[HTML]{F4AC7E}0.812 & \cellcolor[HTML]{F7C190}0.759 & \cellcolor[HTML]{FEEFB5}0.643 & \cellcolor[HTML]{F9CD99}0.729 & \cellcolor[HTML]{F29E72}0.850 & \cellcolor[HTML]{F5B182}0.800 & \cellcolor[HTML]{FCE2AB}0.675 & \cellcolor[HTML]{EF8A62}0.900 & \cellcolor[HTML]{FDEBB2}0.653 & \cellcolor[HTML]{F8CA97}0.737 & \cellcolor[HTML]{F9D09C}0.721 & \cellcolor[HTML]{F8C996}0.738 & \cellcolor[HTML]{F3A477}0.834 & \cellcolor[HTML]{F09269}0.879 & \cellcolor[HTML]{F6B787}0.784 & \cellcolor[HTML]{F5B384}0.795 & \cellcolor[HTML]{F3A477}0.834 \\
\midrule
\multicolumn{28}{l}{\textit{Ours}} \\
Qwen3-VL-4B & Ours & $\mathcal{S}_2\cup\mathcal{P}_1$ & $\mathcal{S}_3$ & \cellcolor[HTML]{F6B989}0.780 & \cellcolor[HTML]{F9CF9B}0.724 & \cellcolor[HTML]{F4AD7F}0.812 & \cellcolor[HTML]{F5B384}0.795 & \cellcolor[HTML]{F08F66}0.887 & \cellcolor[HTML]{F09369}0.877 & \cellcolor[HTML]{FAD49F}0.711 & \cellcolor[HTML]{F9CC98}0.733 & \cellcolor[HTML]{F6BB8B}0.774 & \cellcolor[HTML]{FDEBB2}0.651 & \cellcolor[HTML]{FAD49F}0.711 & \cellcolor[HTML]{F4AE80}0.808 & \cellcolor[HTML]{F8C794}0.745 & \cellcolor[HTML]{FBDAA4}0.696 & \cellcolor[HTML]{F29C71}0.853 & \cellcolor[HTML]{F9CE9A}0.725 & \cellcolor[HTML]{F6BB8A}0.775 & \cellcolor[HTML]{F6B787}0.785 & \cellcolor[HTML]{F7C190}0.759 & \cellcolor[HTML]{F6BA89}0.778 & \cellcolor[HTML]{F1996F}0.861 & \cellcolor[HTML]{F6B787}0.785 & \cellcolor[HTML]{F2A074}0.844 & \cellcolor[HTML]{F29C71}0.853 \\
Qwen3-VL-8B & Ours & $\mathcal{S}_2\cup\mathcal{P}_1$ & $\mathcal{S}_3$ & \cellcolor[HTML]{F6BB8B}0.775 & \cellcolor[HTML]{F9D09C}0.721 & \cellcolor[HTML]{F5B183}0.799 & \cellcolor[HTML]{F8CA96}0.738 & \cellcolor[HTML]{F19A70}0.858 & \cellcolor[HTML]{F2A074}0.843 & \cellcolor[HTML]{FBDDA7}0.688 & \cellcolor[HTML]{F8C593}0.748 & \cellcolor[HTML]{F9CD99}0.728 & \cellcolor[HTML]{FAD5A0}0.709 & \cellcolor[HTML]{FAD49F}0.712 & \cellcolor[HTML]{EF8C63}0.895 & \cellcolor[HTML]{F7C290}0.758 & \cellcolor[HTML]{FBDDA6}0.689 & \cellcolor[HTML]{F2A276}0.839 & \cellcolor[HTML]{F7C18F}0.760 & \cellcolor[HTML]{F4AD7F}0.812 & \cellcolor[HTML]{F7C190}0.759 & \cellcolor[HTML]{FAD8A3}0.700 & \cellcolor[HTML]{F3A77A}0.825 & \cellcolor[HTML]{F2A074}0.844 & \cellcolor[HTML]{F3A377}0.835 & \cellcolor[HTML]{F5B384}0.795 & \cellcolor[HTML]{F19B70}0.857 \\
\midrule
\multicolumn{28}{l}{\textit{Baselines}} \\
Qwen3-VL-8B & raw-distillation SFT & $\mathcal{S}_1\cup\mathcal{P}_1$ & -- & \cellcolor[HTML]{F6BC8B}0.772 & \cellcolor[HTML]{F9D19D}0.719 & \cellcolor[HTML]{F4AF81}0.805 & \cellcolor[HTML]{F5B182}0.801 & \cellcolor[HTML]{F09168}0.882 & \cellcolor[HTML]{F09168}0.882 & \cellcolor[HTML]{F9CF9B}0.723 & \cellcolor[HTML]{FAD49F}0.710 & \cellcolor[HTML]{F8C794}0.745 & \cellcolor[HTML]{FFF6BB}0.625 & \cellcolor[HTML]{F8C492}0.751 & \cellcolor[HTML]{F4AC7E}0.812 & \cellcolor[HTML]{F7C391}0.753 & \cellcolor[HTML]{FBDBA5}0.694 & \cellcolor[HTML]{F29F73}0.846 & \cellcolor[HTML]{FCE4AC}0.670 & \cellcolor[HTML]{F8C593}0.750 & \cellcolor[HTML]{F7BF8D}0.766 & \cellcolor[HTML]{F9CE9A}0.726 & \cellcolor[HTML]{F4AD7F}0.810 & \cellcolor[HTML]{F29D71}0.852 & \cellcolor[HTML]{F8CA97}0.736 & \cellcolor[HTML]{F2A175}0.840 & \cellcolor[HTML]{F3A578}0.831 \\
Qwen3-VL-4B & raw-distillation SFT & $\mathcal{S}_1\cup\mathcal{P}_1$ & -- & \cellcolor[HTML]{F6BD8C}0.771 & \cellcolor[HTML]{FAD39E}0.713 & \cellcolor[HTML]{F3A87A}0.824 & \cellcolor[HTML]{F5B585}0.791 & \cellcolor[HTML]{F09369}0.877 & \cellcolor[HTML]{EF8D65}0.891 & \cellcolor[HTML]{FCE3AB}0.674 & \cellcolor[HTML]{F8C492}0.752 & \cellcolor[HTML]{F7C391}0.754 & \cellcolor[HTML]{FEF1B7}0.637 & \cellcolor[HTML]{FEEEB5}0.644 & \cellcolor[HTML]{F5B283}0.798 & \cellcolor[HTML]{F9CC99}0.730 & \cellcolor[HTML]{FDE6AE}0.665 & \cellcolor[HTML]{F2A175}0.841 & \cellcolor[HTML]{FCE2AB}0.676 & \cellcolor[HTML]{F7C291}0.756 & \cellcolor[HTML]{F8C693}0.748 & \cellcolor[HTML]{F7C290}0.757 & \cellcolor[HTML]{F4AC7E}0.814 & \cellcolor[HTML]{F09168}0.881 & \cellcolor[HTML]{F5B182}0.800 & \cellcolor[HTML]{F3A87B}0.823 & \cellcolor[HTML]{F4AA7C}0.819 \\
Qwen3-VL-8B & RoboReward & -- & -- & \cellcolor[HTML]{FBD9A3}0.698 & \cellcolor[HTML]{F9CF9B}0.723 & \cellcolor[HTML]{F8C896}0.740 & \cellcolor[HTML]{F7C391}0.755 & \cellcolor[HTML]{F3A579}0.830 & \cellcolor[HTML]{F1996E}0.863 & \cellcolor[HTML]{FAD39E}0.713 & \cellcolor[HTML]{FEF0B7}0.639 & \cellcolor[HTML]{FBDAA4}0.697 & \cellcolor[HTML]{FAD5A0}0.708 & \cellcolor[HTML]{E9ECCF}0.517 & \cellcolor[HTML]{F5B384}0.794 & \cellcolor[HTML]{EFEFCA}0.546 & \cellcolor[HTML]{FCE0A9}0.681 & \cellcolor[HTML]{FFF6BB}0.624 & \cellcolor[HTML]{FEF0B6}0.640 & \cellcolor[HTML]{FBDEA7}0.685 & \cellcolor[HTML]{F7BD8C}0.769 & \cellcolor[HTML]{F5F2C5}0.573 & \cellcolor[HTML]{F5B384}0.795 & \cellcolor[HTML]{F9CD99}0.730 & \cellcolor[HTML]{F9CC99}0.731 & \cellcolor[HTML]{FEEDB4}0.647 & \cellcolor[HTML]{FBD9A3}0.699 \\
Qwen3-VL-4B & RoboReward & -- & -- & \cellcolor[HTML]{FCE1AA}0.678 & \cellcolor[HTML]{FAD6A1}0.707 & \cellcolor[HTML]{F7C190}0.759 & \cellcolor[HTML]{F7C391}0.755 & \cellcolor[HTML]{F8C996}0.738 & \cellcolor[HTML]{F6BC8B}0.772 & \cellcolor[HTML]{FBDCA5}0.692 & \cellcolor[HTML]{E6EBD2}0.499 & \cellcolor[HTML]{F9CC99}0.731 & \cellcolor[HTML]{FDECB3}0.651 & \cellcolor[HTML]{E8ECD0}0.512 & \cellcolor[HTML]{F29F73}0.846 & \cellcolor[HTML]{E2E9D6}0.482 & \cellcolor[HTML]{FDE9B0}0.659 & \cellcolor[HTML]{FAD8A2}0.702 & \cellcolor[HTML]{F6F3C4}0.578 & \cellcolor[HTML]{F9CD99}0.729 & \cellcolor[HTML]{FDE9B0}0.658 & \cellcolor[HTML]{F5F2C5}0.573 & \cellcolor[HTML]{F5B686}0.788 & \cellcolor[HTML]{FDEBB2}0.653 & \cellcolor[HTML]{FCE4AC}0.671 & \cellcolor[HTML]{FAD29E}0.715 & \cellcolor[HTML]{F5B686}0.789 \\
\midrule
\multicolumn{28}{l}{\textit{Reference floor (no fine-tune)}} \\
Qwen3-VL-4B & -- & -- & -- & \cellcolor[HTML]{FFF7BC}0.621 & \cellcolor[HTML]{FBDEA7}0.685 & \cellcolor[HTML]{FCE2AB}0.675 & \cellcolor[HTML]{FDECB3}0.649 & \cellcolor[HTML]{FEF1B7}0.638 & \cellcolor[HTML]{F6F2C4}0.576 & \cellcolor[HTML]{EAEDCF}0.520 & \cellcolor[HTML]{E7EBD1}0.507 & \cellcolor[HTML]{FFF6BB}0.625 & \cellcolor[HTML]{FEF4B9}0.631 & \cellcolor[HTML]{FBF5C0}0.600 & \cellcolor[HTML]{F8C593}0.748 & \cellcolor[HTML]{C6DBEF}0.344 & \cellcolor[HTML]{FEF1B7}0.638 & \cellcolor[HTML]{FBDBA5}0.693 & \cellcolor[HTML]{F2F0C8}0.557 & \cellcolor[HTML]{FEEEB5}0.645 & \cellcolor[HTML]{FCE0A9}0.680 & \cellcolor[HTML]{EDEECC}0.535 & \cellcolor[HTML]{FCE4AD}0.670 & \cellcolor[HTML]{FDE6AE}0.664 & \cellcolor[HTML]{F4F2C5}0.570 & \cellcolor[HTML]{F1F0C9}0.552 & \cellcolor[HTML]{F7BE8D}0.768 \\
Qwen3-VL-8B & -- & -- & -- & \cellcolor[HTML]{FBF5BF}0.605 & \cellcolor[HTML]{FCE1AA}0.677 & \cellcolor[HTML]{FEF0B6}0.640 & \cellcolor[HTML]{F9F4C1}0.594 & \cellcolor[HTML]{FEEFB6}0.641 & \cellcolor[HTML]{FEF2B8}0.635 & \cellcolor[HTML]{EDEECC}0.533 & \cellcolor[HTML]{F2F1C7}0.560 & \cellcolor[HTML]{F8F4C2}0.590 & \cellcolor[HTML]{FEF3B9}0.631 & \cellcolor[HTML]{DDE6DA}0.456 & \cellcolor[HTML]{FDE7AF}0.662 & \cellcolor[HTML]{E2E9D6}0.483 & \cellcolor[HTML]{FEEDB4}0.646 & \cellcolor[HTML]{F5F2C5}0.571 & \cellcolor[HTML]{F0F0CA}0.548 & \cellcolor[HTML]{FFF4B9}0.630 & \cellcolor[HTML]{FAD8A2}0.701 & \cellcolor[HTML]{EBEDCE}0.525 & \cellcolor[HTML]{F8C895}0.743 & \cellcolor[HTML]{F7F3C3}0.584 & \cellcolor[HTML]{FBD9A4}0.697 & \cellcolor[HTML]{EFEFCA}0.543 & \cellcolor[HTML]{FBDCA6}0.691 \\
\bottomrule
\end{tabular}
\end{adjustbox}
\end{table*}

%% file: tables/case_study.tex

\newlength{\CaseImgW}
\newlength{\CaseImgH}
\setlength{\CaseImgW}{0.205\textwidth}
\setlength{\CaseImgH}{0.085\textheight}

\newcommand{\CaseImage}[1]{%
  \begingroup
  \setlength{\fboxsep}{0pt}%
  \colorbox{black!4}{%
    \begin{minipage}[c][\CaseImgH][c]{\CaseImgW}
      \centering
      \includegraphics[width=\CaseImgW,height=\CaseImgH,keepaspectratio]{#1}
    \end{minipage}%
  }%
  \endgroup
}

\newcommand{\CaseCard}[2]{%
  \setlength{\fboxsep}{5pt}%
  \fcolorbox{black!18}{white}{%
    \begin{minipage}[t]{0.465\textwidth}
      \centering
      {\small\bfseries #1}\par
      #2
    \end{minipage}%
  }%
}

\newcommand{\correct}[1]{\cellcolor{green!12}\textbf{#1}}
\newcommand{\wrong}[1]{\cellcolor{red!7}#1}

\subsection{Qualitative Case Study}
\label{sec:case-study}

\begin{center}

\begin{tabular}{@{}c@{\hspace{0.8em}}c@{}}
\CaseCard{\textsc{Score-A}: instruction-following score}{
  \begin{tabular}{@{}cc@{}}
    \CaseImage{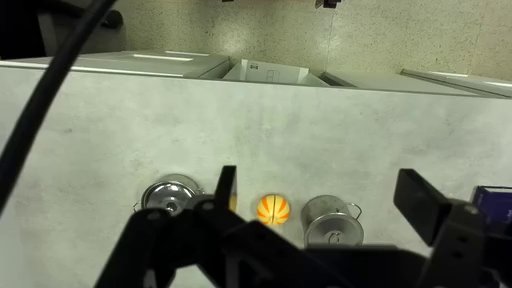} &
    \CaseImage{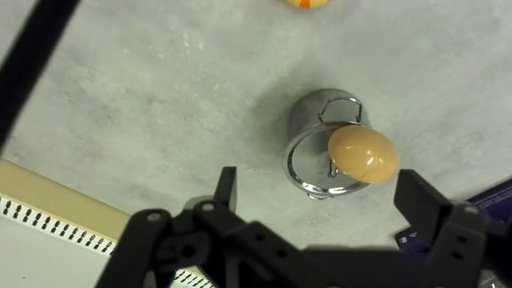} \\
    {\scriptsize first frame} & {\scriptsize final frame}
  \end{tabular}
}
&
\CaseCard{\textsc{Score-B}: VQA answer-quality score}{
  \begin{tabular}{@{}cc@{}}
    \CaseImage{} &
    \CaseImage{} \\
    {\scriptsize first frame} & {\scriptsize final frame}
  \end{tabular}
}
\\[0.8em]
\CaseCard{\textsc{Pair-A}: instruction-following comparison}{
  \begin{tabular}{@{}cc@{}}
    \CaseImage{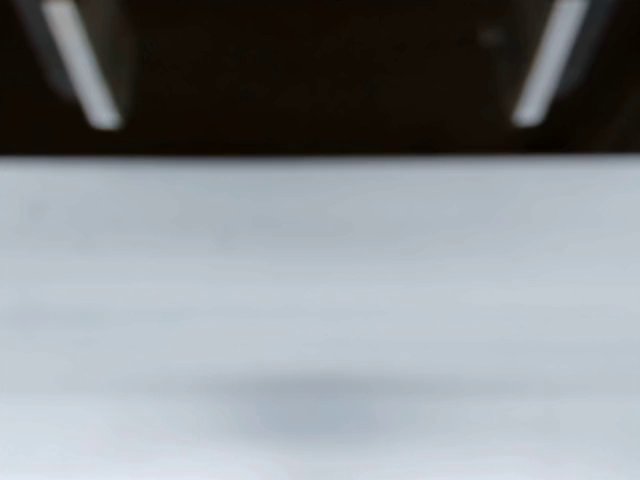} &
    \CaseImage{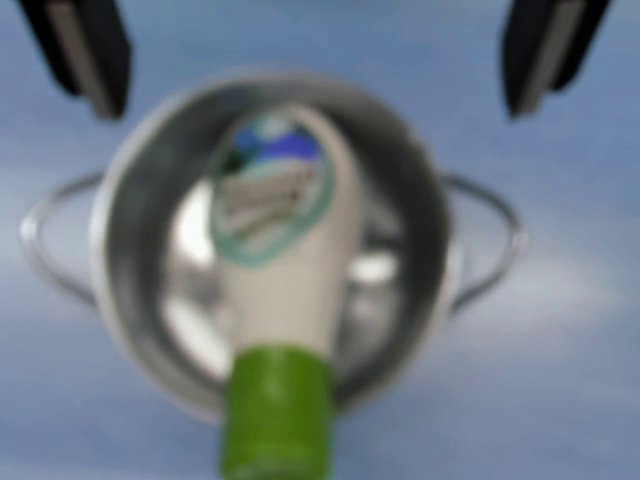} \\
    {\scriptsize video A first} & {\scriptsize video A final} \\[0.25em]
    \CaseImage{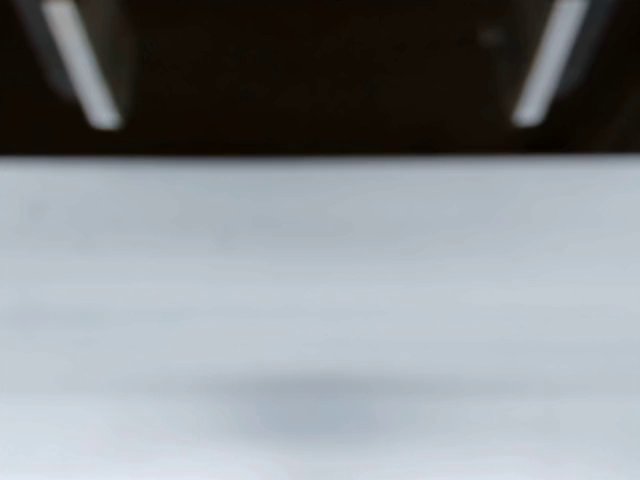} &
    \CaseImage{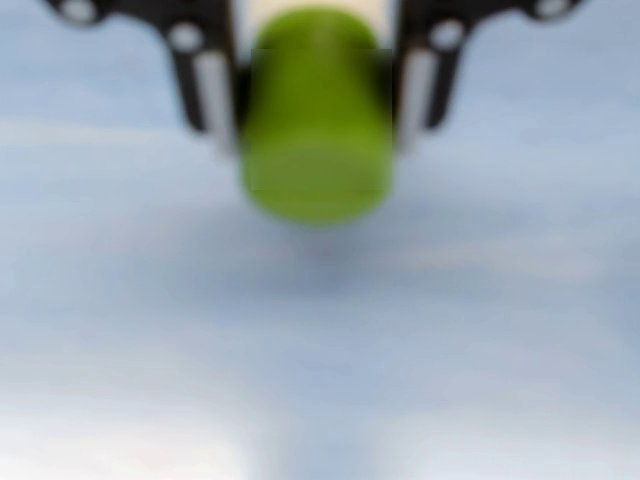} \\
    {\scriptsize video B first} & {\scriptsize video B final}
  \end{tabular}
}
&
\CaseCard{\textsc{Pair-B}: VQA answer-quality comparison}{
  \begin{tabular}{@{}cc@{}}
    \CaseImage{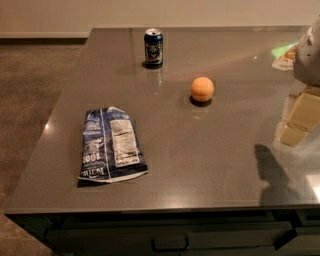} &
    \CaseImage{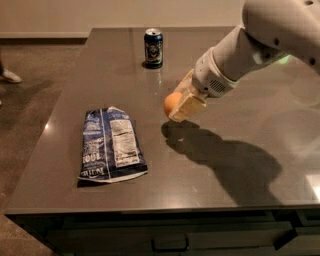} \\
    {\scriptsize first frame} & {\scriptsize final frame}
  \end{tabular}

  {\scriptsize
  Candidate A: move can to front-right corner.\\
  Candidate B: place orange inside blue chip bag.
  }
}
\end{tabular}

\captionof{figure}{\textbf{Qualitative case study across four reward-supervision views.}
We show representative examples from pointwise scoring and pairwise comparison tasks.}
\label{fig:case-study-frames}
\end{center}

\begin{center}
\captionof{table}{\textbf{Model predictions for the qualitative case study.}
Green cells match the gold label, while red cells indicate errors.}
\label{tab:case-study-predictions}

\scriptsize
\setlength{\tabcolsep}{5.2pt}
\renewcommand{\arraystretch}{1.08}
\begin{tabular}{@{}l c c c c c c@{}}
\toprule
Case & GT & Ours & RoboReward-4B & RoboReward-8B & Qwen3-VL-4B & Qwen3-VL-8B \\
\midrule
\textsc{Score-A} & \correct{1} & \correct{1} & \wrong{4} & \wrong{5} & \wrong{5} & \wrong{5} \\
\textsc{Score-B} & \correct{5} & \correct{5} & \wrong{3} & \wrong{3} & \wrong{1} & \wrong{2} \\
\textsc{Pair-A}  & \correct{A} & \correct{A} & \wrong{B} & \wrong{C} & \wrong{C} & \wrong{B} \\
\textsc{Pair-B}  & \correct{B} & \correct{B} & \wrong{C} & \wrong{A} & \wrong{C} & \wrong{C} \\
\bottomrule
\end{tabular}
\end{center}

\noindent\textbf{Analysis.}
Figure~\ref{fig:case-study-frames} and Table~\ref{tab:case-study-predictions} illustrate representative cases where our model aligns its reward prediction with concrete visual evidence, while the baselines often rely on coarse object or action cues. In \textsc{Score-A}, the instruction asks the robot to place the onion into the big pot, but the video shows that the onion remains on the table away from the pot while the robot places an egg on the pot instead. Our model correctly predicts the lowest score ($1$), whereas RoboReward-4B, RoboReward-8B, and untuned Qwen3-VL-4B/8B overestimate the outcome as $4$, $5$, $5$, and $5$, respectively. In \textsc{Score-B}, the candidate answer ``Pick up the yellow cup'' accurately describes the visible robot action, and our model assigns the gold score ($5$), while the baselines underestimate the answer quality. The two pairwise cases show the same pattern in comparative form: in \textsc{Pair-A}, our model selects video A because it makes clearer progress toward placing the ranch bottle into the pot; in \textsc{Pair-B}, it selects candidate B because it is semantically closer to the visible orange and blue chip bag, despite being imperfect. Overall, these examples suggest that our training pipeline improves visual grounding across both pointwise and pairwise reward formats, rather than merely increasing aggregate benchmark scores.